\documentclass[11pt]{article}

\usepackage{dsfont}

\def\colorful{0}

\usepackage{amsmath,amsthm,amsfonts,amssymb,epsfig,color,float,graphicx,verbatim, enumitem}
\usepackage{multirow}
\usepackage{algorithm}
\usepackage[noend]{algpseudocode}
\usepackage{bbm}

\newif\ifhyper\IfFileExists{hyperref.sty}{\hypertrue}{\hyperfalse}
\hypertrue
\ifhyper\usepackage{hyperref}\fi

\usepackage{enumitem}

\usepackage{framed}
\usepackage{nicefrac}

\def\nnewcolor{1}
\ifnum\nnewcolor=1

\fi
\ifnum\nnewcolor=0

\fi

\ifnum\colorful=1
\newcommand{\new}[1]{{\color{red} #1}}
\newcommand{\new}[1]{{\color{red} #1}}
\newcommand{\new}[1]{{\color{blue} #1}}

\newcommand{\newblue}[1]{{\color{blue} #1}}
\else
\newcommand{\new}[1]{{#1}}

\newcommand{\newblue}[1]{{#1}}

\fi

\newtheorem{theorem}{Theorem}[section]
\newtheorem{question}{Question}[section]
\newtheorem{assumption}[theorem]{Assumption}

\newtheorem{lemma}[theorem]{Lemma}
\newtheorem{informal theorem}[theorem]{Theorem (informal statement)}
\newtheorem{condition}[theorem]{Condition}

\newtheorem{corollary}[theorem]{Corollary}
\newtheorem{claim}[theorem]{Claim}
\newtheorem{fact}[theorem]{Fact}

\newtheorem{observation}[theorem]{Observation}
\theoremstyle{definition}
\newtheorem{definition}[theorem]{Definition}
\newcommand{\eqdef}{\stackrel{{\mathrm {\footnotesize def}}}{=}}

\newcommand{\bx}{\mathbf{x}}
\newcommand{\by}{\mathbf{y}}
\newcommand{\bs}{\mathbf{s}}
\newcommand{\bt}{\mathbf{t}}

\newcommand{\bu}{\mathbf{u}}
\newcommand{\bw}{\mathbf{w}}

\newcommand{\tbx}{\tilde{\bx}}
\newcommand{\ty}{\tilde{y}}

\newcommand{\bS}{\mathbf{S}}
\newcommand{\bB}{\mathbf{B}}

\newcommand{\bI}{\mathbf{I}}

\newcommand{\lwe}{\mathrm{LWE}}

\newcommand{\cl}{\mathrm{classic}}

\newcommand{\ns}{\mathrm{noise}}
\newcommand{\nsv}{z}

\newcommand{\massart}{\mathrm{Massart}}
\newcommand{\Dgaus}{D^{\mathcal{N}}}
\newcommand{\Dexp}{D^{\mathrm{expand}}}
\newcommand{\Dcol}{D^{\mathrm{collapse}}}
\renewcommand{\mod}{\mathrm{mod}}
\newcommand{\gaus}{\mathcal{N}}

\newcommand{\alt}{\mathrm{alternative}}
\newcommand{\nul}{\mathrm{null}}
\newcommand{\sr}{\mathrm{SR}}

\newcommand{\PTF}{\mathrm{PTF}}
\newcommand{\LTF}{\mathrm{LTF}}
\newcommand{\sn}{\mathrm{signal}}
\newcommand{\scl}{\mathrm{scale}}
\newcommand{\ad}{\mathrm{add}}
\newcommand{\tru}{\mathrm{truncated}}
\newcommand{\nw}{\mathrm{new}}
\newcommand{\idt}{\mathbf{1}}

\newcommand{\bv}{\mathbf{v}}
\newcommand{\br}{\mathbf{r}}

\newcommand{\R}{\mathbb{R}}
\newcommand{\Z}{\mathbb{Z}}
\newcommand{\N}{\mathbb{N}}
\newcommand{\E}{\mathbf{E}}

\newcommand{\eps}{\epsilon}
\newcommand{\dtv}{d_{\mathrm TV}}
\newcommand{\pr}{\mathbf{Pr}}

\newcommand{\poly}{\mathrm{poly}}

\newcommand{\sgn}{\mathrm{sign}}
\newcommand{\sign}{\mathrm{sign}}

\newcommand{\opt}{\mathrm{OPT}}

\newcommand{\deno}{\mathrm{denominator}}
\newcommand{\numer}{\mathrm{numerator}}

\newcommand{\scldist}[2]{{#1}\circ{#2}}

\title{Cryptographic Hardness of Learning Halfspaces with Massart Noise}

\author{
	Ilias Diakonikolas\thanks{Supported by NSF Medium Award CCF-2107079,
NSF Award CCF-1652862 (CAREER), a Sloan Research Fellowship, and
a DARPA Learning with Less Labels (LwLL) grant.}\\
	UW Madison\\
	{\tt ilias@cs.wisc.edu}\\
	\and
	Daniel M. Kane\thanks{Supported by NSF Medium Award CCF-2107547,
NSF Award CCF-1553288 (CAREER), a Sloan Research Fellowship, and a grant from CasperLabs.}\\
	UC San-Diego \\
	{\tt dakane@ucsd.edu}\\
	\and
	Pasin Manurangsi\\
	Google Research \\
	{\tt pasin@google.com}\\
	\and
	Lisheng Ren\thanks{Supported by NSF Award CCF-1652862 (CAREER) 
	and a DARPA Learning with Less Labels (LwLL) grant.}\\
	UW Madison\\
	{\tt lren29@wisc.edu}\\
}

\begin{document}
\maketitle

\begin{abstract}
We study \new{the complexity of} PAC learning halfspaces in the presence of Massart noise. 
In this \new{problem}, we are given i.i.d.\ labeled examples 
$(\mathbf{x}, y) \in \R^N \times \{ \pm 1\}$, 
where the distribution of $\mathbf{x}$ is arbitrary 
and the label $y$ is a Massart corruption 
of $f(\mathbf{x})$, for an unknown halfspace $f: \R^N \to \{ \pm 1\}$,
with flipping probability $\eta(\mathbf{x}) \leq \eta < 1/2$.
The goal of the learner is to compute a hypothesis with small 0-1 error.
\new{Our main result is the first computational hardness result for this learning problem.
Specifically,} assuming the (widely believed) subexponential-time hardness 
of the Learning with Errors (LWE) problem, we show that no polynomial-time 
\new{Massart halfspace learner} can achieve error better than $\Omega(\eta)$, 
even if \new{the optimal 0-1 error is small, namely}
$\opt  =  2^{-\log^{c} (N)}$
for any universal constant $c \in (0, 1)$. 
Prior work \new{had} provided qualitatively similar evidence of hardness in the
Statistical Query model. Our \new{computational} hardness result 
\new{essentially resolves the polynomial PAC learnability of Massart halfspaces, by showing that}
known \new{efficient} learning algorithms \new{for the problem} are \new{nearly} best possible.
\end{abstract}

\setcounter{page}{0}
\thispagestyle{empty}
\newpage

\section{Introduction} \label{sec:intro}

\subsection{Background and Motivation} \label{ssec:background}

A halfspace or linear threshold function (LTF) 
is any function $h_{\bw, t}: \R^N \to \{\pm 1\}$
of the form  $h_{\bw, t}(\bx) := \sign(\left<\bw, \bx\right> -t )$,
where the vector $\bw \in \R^N$
is called the weight vector, $t \in \R$ is called the threshold,
and $\sgn: \R \to \{\pm 1\}$ is defined by $\sgn(t) = 1$ 
if $t \geq 0$ and $\sgn(t) = -1$ otherwise.
Halfspaces are a central concept class in machine 
learning, extensively investigated since 
the 1950s~\cite{Rosenblatt:58, Novikoff:62, MinskyPapert:68}. 
Here we study the \new{computational} complexity of learning halfspaces in Valiant's 
(distribution independent) PAC model~\cite{val84}, when the labels have
been corrupted by {\em Massart noise}~\cite{Massart2006}.
We define the Massart noise model below.

\begin{definition} [Massart Noise] \label{def:massart-noise-condition}
We say that a joint distribution $D$ of labeled examples $(\mathbf{x},y)$, 
supported on $\R^N\times \{\pm 1\}$,
satisfies the Massart noise condition with noise parameter $\eta \in [0, 1/2)$ 
with respect to a concept class $C$ of Boolean-valued functions on $\R^N$
if there is a concept $c\in C$ such that for all $\mathbf{x}_0\in \R^N$ we have
that $\newblue{\eta(\mathbf{x}_0)} \eqdef 
\pr_{(\mathbf{x},y) \sim D}[c(\mathbf{x})\neq y \mid \mathbf{x}=\mathbf{x}_0]\leq \eta$. 
\end{definition}

\newblue{In other words, a Massart distribution $D$ is a distribution 
over labeled examples $(\mathbf{x}, y) \in \R^N \times \{\pm 1\}$
such that the distribution over examples is arbitrary
and the label $y$ of example $\mathbf{x}$ satisfies 
(i) $y = c(\mathbf{x})$ with probability $1-\eta(\mathbf{x})$,
and (ii) $y = -c(\mathbf{x})$ with probability $\eta(\mathbf{x})$,
for a target function $c\in C$.
Here $\eta(\mathbf{x})$ is an unknown function that satisfies 
$\eta(\mathbf{x}) \leq \eta<1/2$ for all $\mathbf{x}$.
}

The Massart PAC learning problem for the concept class $C$ is the following:\
Given i.i.d.\ samples from a Massart distribution $D$, 
\new{as in Definition~\ref{def:massart-noise-condition},}
the goal is to output a hypothesis with small 0-1 error. 
\new{In this work, we study the computational complexity of the Massart PAC learning problem,
when the underlying concept class $C$ is the class of halfspaces on 
$\R^N$.}

\new{
In its above form, the Massart noise model was defined in~\cite{Massart2006}.
An essentially equivalent noise model had been defined in the 80s by Sloan and 
Rivest~\cite{Sloan88, RivestSloan:94, Sloan96}, 
and a very similar definition had been considered even earlier by Vapnik~\cite{Vapnik82}.

The Massart model is a classical semi-random noise model that
is more realistic than Random Classification Noise (RCN)
\footnote{Random Classification Noise (RCN)~\cite{AL88} is the special case of Massart
noise where the label of each example is independently flipped with probability {\em exactly} $\eta <1/2$.}. 
In contrast to RCN, Massart noise allows for variations in misclassification rates 
(without a priori knowledge of which inputs are more likely to be misclassified).
Asymmetric misclassification rates arise in a number of applications, including
in human annotation noise~\cite{beigman2009learning}.
Consequently, learning algorithms that can tolerate Massart noise are less
brittle than those that depend on the uniformity of RCN.
The agnostic model~\cite{Haussler:92, KSS:94}, where the noise can be fully adversarial, 
is of course even more robust; unfortunately, 
it is computationally hard to obtain agnostic learners with any non-trivial guarantees, 
even for basic settings.

We now return to the class of halfspaces, which is the focus of this work.
We recall that PAC learning halfspaces with RCN 
is known to be solvable in polynomial time (to any desired accuracy)~\cite{BlumFKV96}. 
On the other hand, agnostic PAC learning of halfspaces is known to computationally
hard (even for weak learning)~\cite{GR:06, FGK+:06short, Daniely16}.
} 

The computational task of PAC learning halfspaces corrupted by Massart noise
is a classical problem in machine learning theory that has been posed
by several authors since the 1980s~\cite{Sloan88, Cohen:97, Blum03}.
\new{Until recently, no progress had been made on the efficient PAC learnability 
of Massart halfspaces.}
\cite{DGT19} made the first algorithmic progress on this problem:
they gave a $\poly(N, 1/\eps)$-time learning algorithm 
with error guarantee of $\eta+\eps$. 
\new{Subsequent work made a number of refinements to this algorithmic result,
including giving an efficient proper learner~\cite{CKMY20} and
developing an efficient learner with strongly polynomial sample complexity~\cite{DKT21}.}
\new{In a related direction,}~\cite{DiakonikolasIKL21} gave an efficient boosting algorithm
achieving error $\eta+\eps$ for any concept class,
assuming the existence of a weak learner for the class.

\new{To summarize the preceding discussion, 
all known algorithms for Massart halfspaces achieve error arbitrarily close to $\eta$, 
where $\eta$ is the upper bound on the Massart noise rate.}
The error bound of \new{$\eta$} can be very far from the information-theoretically 
optimum error of \new{$\opt$}, where 
$\opt = R_{\LTF}(D) \leq \eta$. Indeed, 
known \new{polynomial-time} algorithms only guarantee error \new{$\approx \eta$}
even if $\opt$ is very small, i.e., $\opt \ll \eta$.
This prompts the following question:
\begin{question} \label{q:open}
Is there an efficient learning algorithm for Massart halfspaces with
a {\em relative} error guarantee? 
Specifically, if $\opt \ll \eta$ is it possible to achieve error  significantly better than $\eta$?
\end{question}

Our main result \new{(Theorem~\ref{thm:main-inf})} 
answers this question in the negative, assuming the subexponential-time
hardness of the classical Learning with Errors (LWE) problem 
\new{(Assumption~\ref{asm:LWE-hardness})}. 
In other words, {\em we essentially resolve the efficient PAC learnability
of Massart halfspaces}, under a widely-believed cryptographic assumption.

\subsection{Our Results} \label{ssec:results}

\new{Before we state our main result, we recall the setup of} 
the Learning with Errors (LWE) problem.
\new{In the LWE problem,}
we are given samples $(\bx_1, y_1), \dots, (\bx_m, y_m)$ 
and the goal is to distinguish between the following two cases:
\begin{itemize}[leftmargin=*]
\item Each $\bx_i$ is drawn uniformly at random (u.a.r.)\ from $\Z^n_q$, 
and there is a hidden secret vector $\bs \in \Z_q^n$ such that 
$y_i = \left<\bx_i, \bs\right> + \nsv_i$, 
where $\nsv_i \in \Z_q$ is discrete Gaussian noise (independent of $\bx_i$).
\item Each $\bx_i$ and each $y_i$ are independent 
and are sampled u.a.r. from $\Z_q^n$ and $\Z_q$ respectively.
\end{itemize}
Formal definitions of LWE (Definition~\ref{def:general-lwe-one-dim}) 
and related distributions together with the \new{precise computational} 
hardness assumption (Assumption~\ref{asm:LWE-hardness}) \new{we rely on} 
are given in Section~\ref{sec:prelims}. 

Our main result can now be stated as follows:
\begin{theorem}[Informal Main Theorem] \label{thm:main-inf}
Assume that LWE cannot be solved in $2^{n^{1-\Omega(1)}}$ time. 
Then, for any \new{constant} $\zeta > 0$, there is no polynomial-time learning algorithm 
for Massart halfspaces on $\R^N$ that can output a hypothesis with 0-1 
error smaller than $\Omega(\eta)$, even when $\opt \leq 2^{\newblue{-}\log^{1 - \zeta} N}$ 
and the Massart noise parameter $\eta$ is a small positive constant.
\end{theorem} 

\newblue{The reader is also referred to Theorem~\ref{thm:main-formal} in the Appendix
for a more detailed formal statement.}
\new{Theorem~\ref{thm:main-inf} is the first {\em computational} 
hardness result for PAC learning halfspaces (and, in fact, any non-trivial concept class)
in the presence of Massart noise.}
\new{Our result} rules out even 
{\em improper} PAC learning, where the learner is allowed 
to output any polynomially evaluatable hypothesis.
As a corollary, it follows that the algorithm 
given in \cite{DGT19} is essentially the best possible, 
even when assuming that $\opt$ is almost inverse polynomially small 
\new{(in the dimension $N$)}. 
We also remark that this latter assumption is also 
nearly the best possible: if $\opt$ is $o(\eps/N)$, 
then we can just draw $\Omega(N/\eps)$ samples 
and output any halfspace that agrees with these samples.

\new{We note that a line of work has established qualitatively 
similar hardness in the Statistical Query (SQ) model~\cite{Kearns:98}
--- a natural, yet restricted, model of computation.}
Specifically,~\cite{CKMY20} established a super-polynomial SQ lower bound
for learning within error of $\opt+o(1)$.
\new{Subsequently,~\cite{diakonikolas2021near} 
gave a near-optimal super-polynomial SQ lower bound: 
their result rules out the existence of efficient SQ algorithms 
that achieve error better than $\Omega(\eta)$, even if 
$\opt  = 2^{\log^{1 - \zeta} N}$. 
Building on the techniques of~\cite{diakonikolas2021near}, 
more recent work \cite{nassert22} established an SQ lower bound 
for learning to error better than $\eta$, 
even if $\opt  = 2^{\log^{1 - \zeta} N}$ --- matching the guarantees of 
known algorithms exactly.}
 
While the SQ model is quite broad, it is also restricted. 
That is, the aforementioned prior results
do not have any implications for the class of all polynomial-time algorithms.
Interestingly, as we will explain \new{in the proceeding discussion}, 
our computational hardness reduction is inspired by the SQ-hard
instances \new{constructed} in~\cite{diakonikolas2021near}.

\subsection{Brief Technical Overview} \label{ssec:techniques}

Here we give a high-level overview of our approach. 
Our reduction proceeds in two steps. 
The first is to reduce the standard LWE problem (as described above) 
to a different ``continuous'' LWE problem 
more suitable for our purposes. In particular, 
we consider the problem where the $\bx$ samples 
are taken uniformly from $\R^n/\Z^n$, $y$ 
is either taken to be an independent random element of $\R/\Z$ 
or is taken to be $\left<\bx,\bs\right>$ mod $1$ 
plus a small amount of (continuous) Gaussian noise, 
where $\bs$ is some unknown vector in $\{\pm 1\}^n$. 
The reduction between these problems follows 
from existing techniques~\cite{Mic18, vinod2022}.

The second step --- which is the main technical contribution 
of our work --- is reducing this continuous LWE problem 
to that of learning halfspaces with Massart noise. 
The basic idea is to perform a rejection sampling procedure 
that allows us to take LWE samples $(\bx,y)$ 
and produce some new samples $(\tbx,\ty)$. 
We will do this so that if $y$ is independent of $\bx$, 
then $\ty$ is (nearly) independent of $\tbx$; 
but if $y = \left<\bx, \bs\right> + \textrm{noise}$, 
then $\ty$ is a halfspace of $\tbx$  
with a small amount of Massart noise. 
An algorithm capable of learning halfspaces with Massart noise 
(with appropriate parameters) would be able to distinguish 
these cases by learning a hypothesis $h$ 
and then looking at the probability that $h(\tbx) \neq \ty$. 
In the case where $\ty$ was a halfspace with noise, 
this would necessarily be small; 
but in the case where $\tbx$ and $\ty$ were independent, 
it could not be.

In order to manage this reduction, we will attempt 
to produce a distribution $(\tbx, \ty)$ 
similar to the SQ-hard instances of Massart halfspaces 
constructed in~\cite{diakonikolas2021near}. These instances 
can best be thought of as instances of a random variable $(\bx',y')$ 
in $\R^n \times \{\pm 1\}$, where $y'$ is given 
by a low-degree polynomial threshold function (PTF) 
of $\bx'$ with a small amount of Massart noise. 
Then, letting $\tbx$ be the Veronese map applied to $\bx'$, 
we see that any low-degree polynomial in $\bx'$ 
is a linear function of $\tbx$, 
and so $\ty = y'$ is an LTF
of $\tbx$ plus a small amount of Massart noise.

As for how the distribution over $(\bx',y')$ is constructed in~\cite{diakonikolas2021near}, 
essentially the conditional distribution of $\bx'$ on $y' = 1$ and on $y' = -1$ 
are carefully chosen mixtures of discrete Gaussians in the $\bv$-direction 
(for some randomly chosen unit vector $\bv$), 
and independent standard Gaussians in the orthogonal directions.
(\newblue{We will discuss the construction of \cite{diakonikolas2021near} 
in detail in Section~\ref{sec:main-reduction}.})
Our goal will be to find a way to perform rejection sampling 
on the distribution $(\bx,y)$ to produce a distribution of this form.

In pursuit of this, for some small real number
$b$ and some $a \in [0, b)$, we let $\bx'$ be a random Gaussian 
subject to $\bx' \equiv b\bx \pmod b$ (in the coordinate-wise sense) 
conditioned on \new{$by \equiv a \pmod b$}. 
We note that if we ignore the noise in the definition of $y$, 
this implies that 
\new{
$\left<\bx', \bs\right> \equiv \left<b\bx, \bs\right> \equiv b\left<\bx, \bs\right> 
\equiv b y \equiv a \pmod b$ (recalling that $\bs\in \{\pm 1\}^n$).} 
In fact, it is not hard to see that the resulting distribution 
on $\bx'$ is close to a standard Gaussian conditioned 
on $\left<\bx', \bs\right> \equiv a \pmod b$. 
In other words, $\bx'$ is close to a discrete Gaussian 
with spacing
\new{$b/\|\bs\|_2$} and offset \new{$a/\|\bs\|_2$} in the $\bs$-direction, 
and an independent standard Gaussian in orthogonal directions. 
Furthermore, this $\bx'$ can be obtained from $(\bx,y)$ samples by rejection sampling:
taking many samples until one is found with \new{$by \equiv a \pmod b$}, 
and then returning a random $\bx'$ with \new{$\bx' \equiv b\bx \pmod b$}. 
By taking an appropriate mixture of these distributions, 
we can manufacture a distribution close to the hard instances in~\cite{diakonikolas2021near}.
\newblue{This intuition is explained in detail 
in Section \ref{subsec:basic-rejection}; see Lemma \ref{lem:basic-rs}.
(We note that Lemma \ref{lem:basic-rs} is included only for the purposes of intuition; 
it is a simpler version of Lemma \ref{lem:rejection-alternative-distribution}, 
which is one of the main lemmas used to prove our main theorem.)}

Unfortunately, \newblue{as will be discussed in Section \ref{subsec:carving-explaination}}, 
applying this construction directly does not quite work. 
This is because the small noise in the definition of $y$
leads to a small amount of noise in the final values 
of $\langle \bx', \bs  \rangle$. 
This gives us distributions that are fairly similar to the hard instances 
of~\cite{diakonikolas2021near}, but 
leads to small regions of values for $\bu$, 
where the following condition holds: 
$\pr(y'=+1 \mid \bx' = \bu) = \pr(y'=-1 \mid \bx'=\bu)$. 
Unfortunately, the latter condition cannot hold 
if $y'$ is a function of $\bx'$ with Massart noise. 
In order to fix this issue, we need to modify the construction 
by \newblue{carving intervals out of the support of $\bx'$ conditioned on $y'=-1$}, 
in order to eliminate these mixed regions. 
This procedure is discussed in detail in 
Section~\ref{subsubsec:reduction-algorithm}.

\subsection{Prior and Related Work} \label{ssec:related}

\newblue{
In Section~\ref{ssec:background},
we have already summarized the most relevant prior work
on the Massart noise model and in particular on
learning halfspaces with Massart noise. To summarize,
prior to our work, the only known evidence of hardness
for this problem~\cite{CKMY20, diakonikolas2021near, nassert22}
applied for the Statistical Query (SQ) model,
a natural but restricted oracle model of computation.
We reiterate that SQ-hardness does not logically imply hardness for the
class of all polynomial time algorithms.
Our main contribution is the first computational hardness result
for learning Massart halfspaces, which nearly matches the guarantees
of known efficient algorithms~\cite{DGT19, CKMY20, DKT21}.
Our hardness result rules out the existence of any polynomial-time algorithm
for the problem, under a widely believed cryptographic assumption
regarding the hardness of the Learning with Errors (LWE) problem.

For the (much more challenging) agnostic PAC model,
Daniely~\cite{Daniely16} gave a reduction
from the problem of strongly refuting random XOR formulas
to the problem of agnostically learning halfspaces. This hardness result
rules out agnostic learners with error non-trivially smaller than $1/2$ (aka weak learning),
even when the optimal 0-1 error is very close to $0$.
Our hardness result can be viewed as an analogue of Daniely's result 
for the (much more benign) Massart model.

There have also been several recent works showing reductions from LWE 
or lattice problems to other learning problems.
Concurrent and independent work to ours~\cite{Tieg22} showed
hardness of weakly agnostically learning halfspaces,
based on a worst-case lattice problem (via a reduction from ``continuous'' LWE).
Two recent works obtained hardness for the unsupervised problem
of learning mixtures of Gaussians (GMMs), assuming hardness of (variants of) the LWE problem.
Specifically,~\cite{Bruna0ST21} defined a continuous version of LWE (whose hardness they established)
and reduced it to the problem of learning GMMs. More recently,~\cite{vinod2022} obtained a direct
reduction from LWE to a (different) continuous version of LWE; and leveraged this connection
to obtain quantitatively stronger hardness for learning GMMs. It is worth noting that for the purposes
of our reduction, we require as a starting point a continuous version of LWE that differs from the one
defined in~\cite{Bruna0ST21}. Specifically, we require that the distribution
on $\mathbf{x}$ is uniform on $[0, 1]^n$ (instead of a Gaussian, as in~\cite{Bruna0ST21})
and the secret vector is binary. The hardness of this continuous version
essentially follows from~\cite{micciancio2018hardness, vinod2022}.

} 

\newblue{
\paragraph{Organization}
The structure of this paper is as follows: In Section~\ref{sec:prelims}, we record the basic
technical background that will be required throughout this work. In Section~\ref{sec:main-reduction}, 
we give our main hardness reduction that establishes Theorem~\ref{thm:main-inf}. 
In Section~\ref{sec:discussion}, we conclude this work and discuss open problems for future research.
A number of technical proofs have been deferred to the Appendix.

}

\section{Preliminaries} \label{sec:prelims}

\paragraph{Notation and Basic Definitions} 
For $\bx, \bs \in  \R^{n}$ with $\bs \neq \mathbf{0}$, 
let $\bx^{\bs} \eqdef \langle \bx, \bs \rangle / \|\bs\|_2$ be the
length of the projection of $\bx$ in the $\bs$ direction, 
and $\bx^{\perp \bs} \in \R^{n - 1}$ be the 
projection of $\bx$ on the orthogonal
complement of $\bs$. More precisely, let $\bB_{\perp \bs} \in \R^{n\times (n-1)}$ 
for the matrix whose columns form an (arbitrary) orthonormal basis for the orthogonal
complement of $\bs$, 
and let $\bx^{\perp \bs} \eqdef (\bB_{\perp \bs})^T$ \new{$\bx$}.

\newblue{For functions $f, g: U \to \R$, we write $f(u)\propto g(u)$ 
if there is a constant $c\in \R$ such that for all $u\in U$, it holds $f(u)=cg(u)$.} 

We use $X \sim D$ to denote a random variable $X$ with distribution $D$. 
We use $P_D$ or $P_X$ for the corresponding 
probability mass function (pmf) or density function (pdf), and 
$\pr_D$ or $\pr_X$ for the measure function of the distribution.
We use $D_X$ to denote the distribution of the random variable $X$. 
\newblue{For $S\subseteq \R^n$, we will use $\lambda(S)$ 
to denote the $n$-dimensional volume of $S$.} 
Let $U(S)$ denote the uniform distribution on $S$. 
For a distribution $D$ on $\R^n$ and $S \subseteq\R^n$, we denote by
$D \mid S$ the conditional distribution of $X \sim D$ given $X \in S$. 
Let $D^{\bs}$ (resp. $D^{\perp \bs}$) be the distribution 
of $\bx^{\bs}$ (resp. $\bx^{\perp \bs}$), where $\bx \sim D$. 
For distributions $D_1, D_2$, we use $D_1+D_2$ to denote the 
pseudo-distribution with measure function
$\pr_{D_1+D_2}(A)=\pr_{D_1}(A)+\pr_{D_2}(A)$.
Similarly, for $a \in \R$, we use $a \, D$ to denote 
the pseudo-distribution with measure function $a\pr_{D}$.
On the other hand, we use $\scldist{a}{D}$ to denote 
the distribution of $aX$, where $X \sim D$.
We use $D_1\star D_2$ to denote the convolution of distributions $D_1, D_2$.

\newblue{We will use $\LTF_N$ for the class of halfspaces on 
$\R^N$; when $N$ is clear from the context, we may discard it and simply write $\LTF$.}
We will also require the notion of a degree-$d$ PTF.
The class $\PTF_{N,d}$ consists of all functions $h_{p}: 
\R^N \to \{\pm 1\}$ of the form 
$h_{\bw}(\bx) := \sign(p(\bx) )$, where $p: \R^N \to \R$ 
is a degree-$d$ polynomial. 

For $q \in \N$, we use $\Z_q \eqdef \{0,1,\cdots,q-1\}$ 
and $\R_q \eqdef [0,q)$. 
We use $\mod_{q}: \R^n \mapsto \R_q^n$ 
to denote the function that applies $\mod_q(\bx)$ 
on each coordinate of $\bx$.

We use $\Dgaus_{\R^n,\sigma}$ to denote the $n$-dimensional Gaussian distribution
with mean $\mathbf{0}$ and covariance matrix $\sigma^2/(2\pi)\cdot \bI_{n}$
and use $\Dgaus_\sigma$ as a short hand for $\Dgaus_{\R,\sigma}$.
\newblue{In some cases, we will use $\cal{N}(\bf{0}, \bI_{n})$ 
for the standard (i.e., zero mean and identity covariance) multivariate Gaussian,}

\newblue{We will require the following notion of a partially supported Gaussian distribution.}

\begin{definition} [Partially Supported Gaussian Distribution] \label{def:dpart}
For $\sigma\in \R_{+}$ and $\mathbf{x}\in \R^n$, let
$\rho_{\sigma}(\mathbf{x})\eqdef \sigma^{-n}\exp\left (-\pi(\|\mathbf{x}\|_2/\sigma)^2\right )$.
For any countable set\footnote{We will take $S$ to be shifts of lattices, guaranteeing  
that $\rho_{\sigma^2}(S)$ is finite and the distribution is well-defined.} $S \subseteq \R^n$, 
we let $\rho_{\sigma}(S)\eqdef \sum_{\mathbf{x}\in S} \rho_{\sigma}(\mathbf{x})$, 
and let $\Dgaus_{S,{\sigma}}$ be the distribution supported on $S$ with pmf
$P_{\Dgaus_{S,\sigma}}(\mathbf{x}) = \rho_{\sigma}(\mathbf{x})/\rho_{\sigma}(S)$.
\end{definition}

\newblue{
The following definition of a discrete Gaussian will be useful.

\begin{definition}[Discrete Gaussian] \label{def:dg}
For $T\in \R_+, y\in \R $ and $\sigma\in \R_+$, we define the 
``$T$-spaced, $y$-offset discrete Gaussian distribution with $\sigma$ scale'' 
to be the distribution of $D_{T\Z+y,\sigma}^\gaus$.
\end{definition}
}

\vspace{-0.2cm}

\paragraph{Problem Definitions}
\newblue{Here we provide formal definitions of the computational problems we study
together with the computational hardness assumption we rely on.}

\newblue{
Specifically, we define hypothesis testing/decision versions of both the LWE problem
and the Massart halfspace problem. The following terminology will be useful.
We say that an algorithm $A$ solves a hypothesis testing problem 
with $\alpha$ advantage, for some $\alpha>0$, if it satisfies the following requirement:
Let $p_\nul$ (resp. $p_\alt$) be the probability that $A$ outputs ``alternative hypothesis'' 
if the input distribution is from the null hypothesis (resp. alternative hypothesis) case.
Then, we require that $p_\alt-p_\nul \geq \alpha$.}

\medskip

\noindent {\bf Learning with Errors (LWE)}
We use the following definition of LWE, 
which allows for flexible distributions of samples, 
secrets, and noises. Here $m$ is the number of samples, 
$n$ is the dimension, and $q$ is the ring size.

\begin{definition} [Generic LWE]\label{def:general-lwe-one-dim}
Let $m,n,q \in \N$, and let 
$D_\mathrm{sample},D_\mathrm{secret},D_\ns$ 
be distributions on $\R^n, \R^n, \R$ respectively. 
In the $\lwe(m,D_\mathrm{sample},D_\mathrm{secret},D_\ns,\mod_q)$ problem, 
we are given $m$ independent samples $(\mathbf{x},y)$ 
and want to distinguish between the following two cases:
\begin{enumerate}[leftmargin=*]
	\item [(i)] {\bf Alternative hypothesis}: $\bs$ is drawn from $D_\mathrm{secret}$. 
	Then, each sample is generated by taking $\bx\sim D_\mathrm{sample}, \nsv\sim D_\ns$, 
	and letting $y= \mod_q(\left<\bx, \bs\right>+\nsv)$.  	
	\item [(ii)] {\bf Null hypothesis}: $\bx, y$ 
	are independent and each has the same marginal distribution \newblue{as} above.
\end{enumerate}
\end{definition}

When a distribution in LWE 
is uniform over some set $S$, 
we may abbreviate $U(S)$ merely as $S$.
Note that $\lwe(m,\Z_q^n,\Z_q^n,\Dgaus_{\Z,\sigma},\mod_q)$ 
to the classical LWE problem.

\medskip

\noindent {\bf Computational Hardness Assumption for LWE}
As alluded to earlier, the assumption for our hardness result 
is the hardness of the (classic) LWE problem, with the parameters stated below.

\begin{assumption}[Standard LWE Assumption (see, e.g., \cite{lindner2011better})] \label{asm:LWE-hardness}
Let $c > 0$ be a sufficiently large constant. 
For any \newblue{constant} $\beta\in (0,1)$, $\kappa \in \N$, 
$\lwe(2^{O(n^\beta)},\Z_q^n,\Z_q^n ,\Dgaus_{\Z, \sigma},\mod_q)$ 
with $q \leq n^\kappa$ and $\sigma = c\sqrt{n}$ 
cannot be solved in $2^{O(n^\beta)}$ time with $2^{-O(n^\beta)}$ advantage.
\end{assumption}

\newblue{We recall that} \cite{regev2009lattices,Peikert09} 
gave a polynomial-time \emph{quantum} reduction
\footnote{The reduction in~\cite{regev2009lattices} 
uses \emph{continuous} Gaussian noise; such continuous noise 
can be post-processed to be discrete Gaussian~\cite{Peikert09}.} 
from approximating (the decision version of) 
the Shortest Vector Problem (GapSVP) to LWE (with similar $n, q, \sigma$ parameters). 
\newblue{Our hardness assumption is the widely believed 
\emph{sub-exponential} hardness of LWE. We note that} 
the fastest known algorithm for GapSVP takes $2^{O(n)}$ time~\cite{AggarwalLNS20}.
Thus, refuting the conjecture would be a major breakthrough. 
A similar assumption was also used in~\cite{vinod2022} 
to establish computational hardness of learning Gaussian mixtures. 
Our use of a sub-exponential hardness of LWE is not a coincidence; 
see Section \ref{sec:discussion}.

As mentioned earlier, we will use a \newblue{different variant} of LWE, 
where the sample is from $\R_1^n$, the secret is from $\{\pm 1\}^n$, 
and the noise is drawn from a continuous Gaussian distribution. The 
hardness of this variant is stated 
\newblue{in Lemma~\ref{lem:continuous-lwe-hardness}} below. 

\begin{lemma} \label{lem:continuous-lwe-hardness}
Under Assumption \ref{asm:LWE-hardness}, for any $\beta\in (0,1)$ and $\gamma \in \R_+$, 
there is no $2^{O(n^{\beta})}$ time algorithm to solve $\lwe\left (2^{O(n^{\beta})},\R_1^n,\{\pm 1\}^n, 
\Dgaus_{O(n^{-\gamma})},\mod_1\right )$ with $2^{-O(n^{\beta} )}$ advantage.
\end{lemma}

The proof of Lemma~\ref{lem:continuous-lwe-hardness}, 
which follows from~\cite{Mic18, vinod2022}, 
is deferred to Appendix \ref{app:lwe-reduction}.

\medskip

\noindent {\bf Decision Version of Massart Halfspace Problem}
\new{Before we define the testing version of the Massart halfspace problem, 
we introduce the following notation which will be useful throughout. 

For a distribution $D$ on labeled examples $(\mathbf{x},y)$ and a concept class $C$,}
we will use $R_C(D) \eqdef \min_{h\in C}\pr_{(\mathbf{x},y) \sim D} [h(\mathbf{x})\neq y]$ 
to denote the error of the best classifier in $C$ with respect to $D$, 
and $R_{\opt}(D) \eqdef \min_{h:\R^N \to \{\pm 1\}}\pr_{(\mathbf{x},y) \sim D} [h(\mathbf{x})\neq y]$ 
for the error of the Bayes optimal classifier. \new{Note that if the distribution $D$ is a Massart distribution
(corresponding to the class $C$), then we have that 
$R_C(D) = R_{\opt}(D) = \E_{\mathbf{x} \sim D_{\mathbf{x}}} [\eta(\mathbf{x})]$, 
where $D_{\mathbf{x}}$ is the marginal distribution of $D$ on $\mathbf{x}$.}

We will prove hardness for the following
decision version of learning Massart halfspaces. 
\newblue{This will directly imply hardness for the corresponding learning (search) problem.}

\begin{definition} [Testing Halfspaces with Massart Noise] \label{def:test-Massart}
For $n,N\in \N$, $\eta,\opt\in (0,1/2)$, 
let \\ $\massart(m,N,\eta,\opt)$ denote the problem 
of distinguishing, given $m$ i.i.d.\ samples from $D$ 
on $\R^N\times \{\pm 1\}$, between the following two cases:
\begin{enumerate}[leftmargin=*]
	\item [(i)] {\bf Alternative hypothesis}: $D$ satisfies the Massart halfspace condition 
with noise parameter \newblue{$\eta$} and $R_\LTF(D)\leq\opt$; and		
	\item [(ii)] {\bf Null hypothesis}: the Bayes optimal classifier has $c\eta$ error, 
	where $c>0$ is a sufficiently small universal constant.
\end{enumerate}
\end{definition}

\section{Reduction from LWE to Learning Massart Halfspaces}
\label{sec:main-reduction}

\newblue{In this section, we establish Theorem~\ref{thm:main-inf}.
Some intermediate technical lemmas have been deferred to the Appendix.
Our starting point is the problem $\lwe(m,\R_1^n,\{\pm 1\}^n,\Dgaus_{\sigma},\mod_1)$.
Note that, by Lemma~\ref{lem:continuous-lwe-hardness}, 
\new{Assumption~\ref{asm:LWE-hardness}} implies the hardness of 
$\lwe(m,\R_1^n,\{\pm 1\}^n,\Dgaus_{\sigma},\mod_1)$. 
We will reduce this variant of LWE to the decision/testing version of Massart halfspaces 
(Definition~\ref{def:test-Massart}).}

\new{Our reduction will employ multiple underlying parameters, 
which are required to satisfy a set of conditions. 
For convenience, we list these conditions below.}

\new{

\begin{condition}\label{cond:params}
Let $n,m,m'\in \N$, $t,\epsilon,\sigma\in\R_{+}$, $\delta\in (0,1)$, 
satisfy: 
\begin{enumerate}
	\item [(i)] $t/\epsilon$ is a sufficiently large even integer, 
	\item [(ii)] $\sigma\leq \sqrt{n}$, 
	\item [(iii)] $\frac{1}{t\sqrt{n}}\geq \sqrt{c\log (n/\delta)}$, 
	where $c$ is a sufficiently large universal constant,
	\item [(iv)] $(\frac{c'\epsilon}{c''t\sigma})^2\geq \log (m'/\delta)$, 
	where $c'>0$ is a sufficiently small universal constant 
	and $c'' >0$ is a sufficiently large universal constant.
\end{enumerate}
\end{condition}
The main theorem of this work is stated below.

\newblue{
\begin{theorem} \label{thm:lwe-to-massart}
Let $n, m, m' \in \N$, $t, \eps, \sigma \in \R_+$, $\eps', \delta\in (0,1)$
satisfy Condition~\ref{cond:params} and $\eta<1/2$.
Moreover, assume that $m'=c(\epsilon/t)m$, 
where $c>0$ is a sufficiently small universal constant and $m(\epsilon/t)^2$  is sufficiently large, 
and $N=(n+1)^{d}$, where $d/(t/\eps)$ is sufficiently large.
Suppose that there is no $T+ \poly(m, N, \log(1/\delta))$-time algorithm for solving 
$\lwe(m,\R_1^n,\{\pm 1\}^n,\Dgaus_{\sigma},\mod_1)$ with $\epsilon'-O(\delta)$ advantage. 
Then there is no $T$ time algorithm for solving $\massart(m',N,\eta,\opt)$ with $2\epsilon'$ advantage, 
where $\opt=\exp(-\Omega(t^4/\epsilon^2))$.
\end{theorem}
}
}

Before we continue, let us note that Theorem~\ref{thm:lwe-to-massart}, combined 
with Lemma \ref{lem:continuous-lwe-hardness}, can be easily used 
to prove \new{Theorem~\ref{thm:main-inf}}
(e.g., by plugging in $t = n^{-0.5 - \Theta(\zeta)}$ and $\eps = \Theta(n^{-1.5})$ \new{in the above statement}); 
\newblue{see Appendix~\ref{app:proof-main-thm}}. 
As such, we devote the remainder of the body of this paper
to give an overview to the proof of Theorem~\ref{thm:lwe-to-massart}. 

\vspace{-0.3cm}

\paragraph{High-level Overview}
The starting point of our computational hardness reduction 
is the family of SQ-hard instances obtained in~\cite{diakonikolas2021near}.
At a high-level, these instances are constructed using
mixtures of ``hidden direction'' discrete Gaussian distributions, i.e.,
distributions that are discrete Gaussians in a hidden direction 
and continuous Gaussians on the orthogonal directions (see below for an overview).

In Section~\ref{subsec:basic-rejection}, we note 
that by using an appropriate rejection sampling procedure on 
the LWE samples (drawn from the alternative hypothesis), 
we obtain a distribution very similar to the ``hidden direction discrete Gaussian''. 
A crucial difference in our setting is the existence of 
a small amount of additional ``noise''. 
\new{A natural attempt is to replace the discrete Gaussians 
in \cite{diakonikolas2021near} with the noisy ones obtained 
from our rejection sampling procedure. 
This produces problems similar to the hard instances from \cite{diakonikolas2021near}.}
\new{Unfortunately, the extra noise in our construction means 
that the naive version of this construction will not work; 
even with small amounts of noise, the resulting distributions 
will {\em not} satisfy the assumptions of a PTF with Massart noise.}
In Section~\ref{subsec:carving-explaination}, we elaborate
on \new{this issue and the modifications we need to make to our construction in order to overcome it}.
In Section~\ref{subsec:complete-construction}, we provide 
the complete construction of our Massart PTF hard instance.

\vspace{-0.3cm}

\paragraph{Overview of the \cite{diakonikolas2021near} SQ-hard Construction}
\cite{diakonikolas2021near} showed SQ-hardness for the following 
hypothesis testing version of the problem (which implies hardness for the learning problem):
For an input distribution $D$ on $\R^n\times \{\pm 1\}$, distinguish 
between the cases where $D$ is \new{a specific distribution} 
$D_{\nul}$ \new{in which $\bx$ and $y$ are independent}  or 
where $D$ belongs to a class of alternative hypothesis distributions $\mathcal{D}_{\alt}$.
In particular, for $D\in\mathcal{D}_\alt$, $y$ will be given by a low-degree PTF 
in $\bx$ with a small amount of Massart noise.
As we will be trying to reproduce it, it is important for us 
to understand this alternative hypothesis distribution.
Each distribution in $\mathcal{D}_\alt$ is parameterized 
by a hidden direction $\bs\in S^{n-1}$. 
We will denote the corresponding distribution by $D_\bs$. 
$D_\bs$ is constructed so that 
$\bx^{\perp \bs}\sim \Dgaus_{\R^{n-1},1}$ is independent of $\bx^\bs$ and $y$. 
\new{This means that we can specify $D_\bs$ by describing 
the simpler distribution of $(\bx^\bs,y)\in\R\times\{\pm 1\}.$} 
For $(\bx^\bs,y)$, we have that $y=+1$ with probability $1-\eta$.
The distributions of $\bx^\bs$ conditioned on $y= \pm 1$ are defined 
to be mixtures of discrete Gaussians as follows:
\begin{equation} \label{eqt:dk21-construction}
D_{\bx^\bs\mid (y=+1)}=\frac{1}{\eps}\int_{0}^{\eps}\Dgaus_{u+(t+u)\Z,1} du \; \textrm{ and } \;
D_{\bx^\bs\mid (y=-1)}=\frac{1}{\eps}\int_{t/2}^{t/2+\eps}\Dgaus_{u+(t+u-t/2)\Z,1} du \;.
\end{equation}
As we noted, both $\bx^\bs\mid (y=+1)$ and 
$\bx^\bs\mid (y=-1)$ are mixtures of discrete Gaussians. 
Combining this with the fact that $\bx^{\perp \bs}\sim \gaus(n,\bI_{n-1})$, this indicates that
$\bx\mid (y=+1)$ and $\bx\mid (y=-1)$ are mixtures of ``hidden direction discrete Gaussians''
--- with different spacing and offset for their support on the hidden direction. 
\new{These conditional distributions were carefully selected 
to ensure that $y$ is a Massart PTF of $\bx$ with small error. 
To see why this is,}
first, notice that the support of $\bx^\bs\mid (y=+1)$ is 
$\bigcup\limits_{i\in \Z} [it,it+(i+1)\eps]$, while
the support of $\bx^\bs\mid (y=-1)$ is 
$\bigcup\limits_{i\in \Z} [it+t/2,it+t/2+(i+1)\eps]$;
both supports are unions of intervals (indexed by $i$).
Consider \new{the implications of this for three different ranges of $\bx^\bs$}:
\begin{enumerate}[leftmargin=*]
    \item For $\bx^{\bs}\in \newblue{[-t^2/(2\eps),t^2/(2\eps)]}$, the intervals have lengths 
    in $[0,t/2]$; thus, the $+1$ intervals and the $-1$ intervals do not overlap at all.
	\item \label{itm:carving-explanation} 
	For $\bx^{\bs}\in \newblue{[-t^2/\eps,-t^2/(2\eps))\cup (t^2/(2\eps),t^2/\eps]}$, 
	the intervals have lengths in $[t/2,t]$;
	thus, the $+1$ intervals and the $-1$ intervals overlap, so that 
	their \new{union covers the space}. \new{We note that in this case} 
	there are gaps between the $+1$ intervals; 
	specifically, there are at most $O(t/\eps)$ such gaps.
	\item For $\bx^{\bs}\in \newblue{(-\infty,-t^2/\eps)\cup (t^2/\eps,\infty)}$, 
	the intervals have lengths in $[t,\infty)$,
	so the $+1$ intervals cover the space by themselves.
\end{enumerate}
Consider the degree-$O(t/\eps)$ PTF $\sgn(p(\bx))$ such that  
$\sign(p(\bx))= +1$ iff $\bx^\mathbf{s}\in \bigcup\limits_{i\in \Z} [it,it+(i+1)\eps]$.
\new{In particular, $\sgn(p(\bx))=1$ for $\bx$ in the support 
of the conditional distribution on $y=1$.}
Note that the PTF $\sgn(p(\bx))$ has zero error in the first case; 
thus, its \new{total 0-1} error is at most $\exp(-\Omega(t^2/\eps)^2)$.
Moreover, \new{since the probability of $y=1$ is substantially larger 
than the probability of $y=-1$, it is not hard to see that 
for any $x$ with $\sgn(p(x))=1$ that 
$\pr[y=1\mid \bx=x] > 1-O(\eta).$ 
This implies that $y$ is given by $\sgn(p(x))$ with Massart noise $O(\eta)$}. 

\subsection{Basic Rejection Sampling Procedure} \label{subsec:basic-rejection}
In this subsection, we show that by performing rejection sampling on LWE samples, 
one can obtain a distribution similar to the ``hidden direction discrete Gaussian''.
For the sake of intuition, we start with the following simple lemma.
\newblue{The lemma essentially states that, doing rejection sampling on LWE samples, 
gives a distribution with the following properties:
On the hidden direction $\bs$, the distribution is pointwise close 
to the convolutional sum of a discrete Gaussian and 
a continuous Gaussian noise. Moreover, on all the other directions $\perp \bs$, 
the distribution is nearly independent of its value on $\bs$, in the sense that
conditioning on any value on $\bs$, the distribution on $\perp \bs$ 
always stays pointwise close to a Gaussian. }
We note that this distribution closely resembles 
the ``hidden direction discrete Gaussian'' in~\cite{diakonikolas2021near}.

\begin{lemma} \label{lem:basic-rs}
Let $(\bx,y)$ be a sample of the 
$\lwe(m,\R_1^n,\{\pm 1\}^n,\Dgaus_{\sigma},\mod_1)$ 
from the alternative hypothesis case, 
let $y'$ be any constant in $[0,1)$, and let
$\bx'\sim \scldist{(1/\sigma_\scl)}{\, \Dgaus_{\newblue{\bx}+\Z^n,\sigma_\scl}} \mid (y=y') \;.$
Then we have the following: 
\begin{enumerate}[leftmargin=*]
	\item [(i)] For $\bx'^\bs$, we have that for any $u\in \R$ it holds that
	\newblue{
	$$P_{\bx'^\bs}(u)=(1\pm O(\delta))P_{D'\star \Dgaus_{\sigma_\ns}}(u) \;,$$
	where
	$D'=\Dgaus_{T(y'+\Z),\sigma_\sn}$,
and $T=\sr /(n^{1/2}\sigma_\scl)$,  $\sigma_\sn=\sqrt{\sr}$, $\sigma_\ns=\sqrt{1-\sr}$, 
\new{and $\sr=\frac{\sigma_\scl^2}{\sigma_\scl^2+\sigma^2/n}$.}}

\item [(ii)] $\bx'^{\perp\bs}$ is ``nearly independent'' of $\bx'^\bs$,
namely for any $l\in \R$ and $\bu\in \R^{n-1}$ we have that
$$P_{\bx^{\perp \bs}|\bx^\bs=l}(\bu)
=(1\pm O(\delta)) P_{\Dgaus_{\R^{n-1},1}}(\bu)\; .$$
\end{enumerate}
\end{lemma}

Lemma~\ref{lem:basic-rs} is a special case 
of Lemma~\ref{lem:rejection-alternative-distribution}, 
which is \new{one of the main lemmas} required for our proof.
We note that the distribution of $\bx'$ 
\new{obtained from the above rejection sampling} 
is very similar to the ``hidden direction discrete Gaussian'' 
used in \cite{diakonikolas2021near}. 
The key differences are as follows: 
(i) on the hidden direction, $\bx'^\bs$ is close to a discrete Gaussian 
plus extra Gaussian noise (instead of simply being a discrete Gaussian); 
and (ii) $\bx'^{\perp\bs}$ and $\bx'^\bs$ are not perfectly independent.
More importantly, by taking different values for $y'$ and $\sigma_\scl$, 
we can obtain distributions with the same hidden direction, 
but their discrete Gaussian on the hidden direction 
has different spacing ($T$) and offset ($y'$).

\new{To obtain a computational hardness reduction, 
our goal will be to simulate the instances from \cite{diakonikolas2021near} 
by replacing the hidden direction discrete Gaussians 
with the noisy versions that we obtain from this rejection sampling. 
We next discuss this procedure and see why a naive implementation 
of it does not produce a $\PTF$ with Massart noise.}

\subsection{Intuition for the Hard Instance} 
\label{subsec:carving-explaination}

\new{The natural thing to try is to simulate 
the conditional distributions 
from \cite{diakonikolas2021near} by replacing 
the hidden direction discrete Gaussian terms in \eqref{eqt:dk21-construction} 
with similar distributions obtained from rejection sampling. 
In particular, Lemma \ref{lem:basic-rs} says that 
we can obtain a distribution which is close to this hidden direction Gaussian 
plus a small amount of Gaussian noise. 
Unfortunately, this extra noise will cause problems 
for our construction.}

Recall that the support of $\bx^\bs\mid (y=+1)$ was
$\bigcup\limits_{i\in \Z} [it,it+(i+1)\eps]$, 
\new{and the support of $\bx^\bs\mid (y=-1)$ was $\bigcup\limits_{i\in \Z} [it+t/2,it+t/2+(i+1)\eps]$ for~\cite{diakonikolas2021near}.} 
With the extra noise, there is a decaying \new{density} 
tail in both sides of each \new{$[it,it+(i+1)\eps]$} 
interval in the support of $\bx^\bs\mid (y=+1)$.
The same holds for each interval in the support of 
\new{$\bx^\bs\mid (y=-1)$}.
\new{
Recalling the three cases of these intervals discussed earlier,
this leads to the following issue.
In the second case, the intervals have length within $[t/2,t]$; 
thus, the intervals $[it,it+(i+1)\eps]$ and $[it+t/2,it+t/2+(i+1)\eps]$
overlap, i.e., $it+(i+1)\eps\geq it+t/2$.
\newblue{On the right side of $[it,it+(i+1)\eps]$, in the support of $\bx^\bs\mid (y=-1)$, 
there is a small region of values for $u$, 
where $\pr[y'=+1 \mid \bx^\bs = u] = \pr[y'=-1 \mid \bx^\bs=u]$.}
This causes the labels $y=+1$ and $y=-1$ 
to be equally likely over that \newblue{small region}, 
violating the Massart condition. 
\newblue{(We note that for the first case, there is also this kind 
of small region that $\pr[y'=+1 \mid \bx^\bs = u] = \pr[y'=-1 \mid \bx^\bs=u]$ 
caused by the noise tail. However, the probability density of this region is negligibly small, 
as we will later see in Lemma \ref{lem:ptf-alt-distr}.)}
}

We can address this issue by carving out empty spaces in the 
$[it+t/2, it+t/2+(i+1)\eps]$ intervals 
for $\bx^\bs\mid (y=-1)$, 
so that these decaying parts can fit into.
Since this only needs to be done for intervals 
of Case~\ref{itm:carving-explanation}, 
at most $O(t/\eps)$ many such slots are needed. 
\new{It should be noted that no finite slot will totally prevent 
this from occurring. However, we only need the slot to be wide enough 
so that the decay of the error implies that 
there is only a negligible amount of mass in the overlap (which can be treated as an error).}

We also need to discuss another technical detail.
In the last section, we defined the rejection sampling process as taking
$\scldist{(1/\sigma_\scl)}{\Dgaus_{\bx+\Z^n,\sigma_\scl}} \mid (y=y')$,
where we can control the offset by $y'$ and spacing by $\sigma_\scl$ (Lemma \ref{lem:basic-rs}). 
\newblue{
This distribution produced in Lemma \ref{lem:basic-rs} 
is effectively a noisy version of a discrete Gaussian. 
Therefore, we can produce a noisy version of the hard instances 
of \cite{diakonikolas2021near} by taking a mixture of these noisy discrete Gaussians. 
Unfortunately the noise rate of one of these instances will be $\sigma_\ns$.
This quantity depends on the spacing $T$ of the discrete Gaussian, 
which varies across the mixture we would like to take. 
This inconsistent noise rate is inconvenient for our analysis. 
However, we can fix the issue by adding extra noise artificially 
to each of the discrete Gaussians in our mixture, 
so that they will all have a uniform noise rate $\sigma_\ns$; 
see Algorithm \ref{algo:sampling} and Lemma \ref{lem:rejection-alternative-distribution}.}

The last bit of technical detail is that instead of doing the rejection for $y=y'$, 
which has $0$ acceptance probability, we will only reject 
if $y$ is not corresponding to any discrete Gaussian we need. 
Then we do another rejection to make sure that the magnitude 
of discrete Gaussians in the mixture is correct. 
In the next subsection, we introduce the complete rejection sampling method.

\subsection{The Full Hard Instance Construction}\label{subsec:complete-construction}

We first introduce the complete rejection algorithm, 
and then explain how the hard instance is produced using it.
Below we provide proof overviews; 
omitted proofs can be found in Appendix \ref{app:main-reduction}.

\subsubsection{The Complete Rejection Algorithm} \label{subsubsec:complete-rejection}
The rejection sampling algorithm is the following. 
The sampling process produces the noisy variant of the 
distribution which, for some carefully selected set $B\subseteq [0,1]$, 
has PDF function $\frac{1}{\lambda(B)}\int_B \Dgaus_{k+(t+k-\psi)\Z,1} dk$ 
in the hidden direction, as we will see in Lemma~\ref{lem:rejection-alternative-distribution}. 

\begin{algorithm}[h!]
\caption{Rejection Sampling Algorithm} \label{algo:sampling}

\textbf{Inputs:} 
A sample $(\bx, y) \in \R_1^n \times \R_1$ 
and the input parameters are
$t,\epsilon, \psi\in\R_{>0}$, 
where $\psi+\eps\leq t$, $B\subseteq [\psi, \psi+\epsilon]$, $\delta\in (0,1)$.
In addition, \new{the parameters satisfy items (i)-(iii)
of Condition~\ref{cond:params}.}

\textbf{Output:} REJECT or a sample $\bx' \in \R^n$.

\begin{enumerate}[leftmargin=*]
	\item Reject unless there is a $k\in B$ such that $y=\frac{k}{t+k-\psi}$. \label{lne:first-rejection}
	
	\item Furthermore, reject with probability $1-\frac{t^2}{(t+k-\psi)^2}$. \label{lne:second-rejection}
		
	\item Let $\sr = 1 - 4(t + \eps)^2\sigma^2, \sigma_\scl = \frac{\sr}{(t + k - \psi)\sqrt{n}}$ and $\sigma_\ad = \sqrt{\frac{(1 - \sr)\sigma_\scl^2 - \sr(\sigma/\sqrt{n})^2}{\sr}}$. 
	Then, sample independent noise $\bx_\ad\sim \Dgaus_{\R^n,\sigma_\ad}$
	and output
	$\bx'\sim \scldist{(1/\sigma_\scl)}{\Dgaus_{\bx+\bx_\ad+\Z^n,\sigma_\scl}}$.	
	\label{lne:output-sampling}
\end{enumerate}
\end{algorithm}

\smallskip

Notice that the parameter $\sr$ does \emph{not} depend on $y$, 
whereas $\sigma_\scl, \sigma_\ad$ do depend on $y$. 

\noindent For convenience, let us use the following notation for the output distributions. 

\begin{definition}[Output Distribution of Rejection Sampling]
Let $D_{t,\eps,\psi,B,\delta}^\alt$ be the distributions of $\bx'$ 
produced by Algorithm \ref{algo:sampling} 
(conditioned that the algorithm accepts\footnote{For brevity, we say that the algorithm \emph{accepts} if it does not reject.}) 
given that $(\bx, y)$ are sampled \new{as follows}: 
let $\bx \sim U(\R_1^n), \nsv \sim \Dgaus_\sigma$, 
and then let $y = \mathrm{mod}_1(\left<\bx, \bs\right> + \nsv)$, 
where $\bs \in \{\pm 1\}^n$ is the secret.
Furthermore, let $D_{t,\eps,\psi,B,\delta}^\nul$ be a similar distribution, 
but when $\bx \sim U(\R_1^n), y \sim U(\R_1)$ are independent. 
\end{definition}

Note that $D_{t,\eps,\psi,B,\delta}^\alt$ depends on $\bs$, 
but we do not explicitly denote this in our notation.

\paragraph{Alternative Hypothesis Analysis}
The main properties of $D_{t,\epsilon,\psi,B,\delta}^\alt$ are summarized in the following lemma. 
\newblue{
Essentially, the lemma states that for this distribution $D_{t,\eps,\psi,B,\delta}^\alt$, 
the marginal distribution on the hidden direction $\bs$ is pointwise close 
to the convolution sum of $D'$ and a Gaussian noise, 
where $D'$ is a linear combination of discrete Gaussians. 
Moreover, on all the other directions $\perp \bs$, the distribution 
is nearly independent of its value on $\bs$, in the sense that
conditioning on any value on $\bs$, the distribution on $\perp \bs$ 
always stays pointwise close to a Gaussian.   
}

\begin{lemma} \label{lem:rejection-alternative-distribution}
Let 
$\bx'\sim D_{t,\epsilon,\psi,B,\delta}^\alt$. 
Then we have the following:
\begin{enumerate}[leftmargin=*]
	\item [(i)] For $\bx'^\bs$, we have that for any $u\in \R$
	\newblue{
	$$P_{\bx'^\bs}(u)=(1\pm O(\delta))P_{D'\star \Dgaus_{\sigma_\ns}}(u)\; ,$$
	where
	$$D'=\frac{1}{\lambda(B)}\int_B \Dgaus_{k+(t+k-\psi)\Z,\sigma_\sn} dk\; ,$$
	$\sigma_\sn=\sqrt{\sr}$, and
	$\sigma_\ns=\sqrt{1-\sr} = 2(t+\eps)\sigma$. ($\sr$ is defined in Algorithm \ref{algo:sampling}.)}
		
	\item [(ii)] 
	$\bx'^{\perp\bs}$ is ``nearly independent'' of $\bx'^\bs$; namely, for any $l\in \R$ and $\bu\in \R^{n-1}$, we have that
	$$P_{\bx'^{\perp \bs}\mid\bx'^\bs=l}(\bu)
	=(1\pm O(\delta)) P_{\Dgaus_{\R^{n-1},1}}(\bu)\;.$$
\end{enumerate}    
\end{lemma}

The crux of the proof is the observation that the two following processes 
of generating $(\bw, \newblue{y})$ result in the same (joint) distribution:
\begin{itemize}[leftmargin=*]
\item Sample $\bx \sim U(\R_1^n), \bx_\ad\sim \Dgaus_{\R^n,\sigma_\ad}, \nsv \sim \Dgaus_{\sigma}$ and let $\bw \sim \scldist{(1/\sigma_\scl)}{\Dgaus_{\bx+\bx_\ad+\Z^n,\sigma_\scl}}, y = \mod_1(\langle \bx,\mathbf{s}\rangle +\nsv)$. 
\item Sample $\bx_\nw \sim U(\R_1^n), \nsv_\nw\sim \Dgaus_{\sqrt{\sigma^2+\|\mathbf{s}\|_2^2\sigma_{\mathrm{add}}^2}}$ 
and let $\bw \sim \scldist{(1/\sigma_\scl)}{\Dgaus_{\bx_\nw+\Z^n,\sigma_\scl}}, 
y = \mod_1(\sigma_\scl\|\bs\|_2\bw^{\bs} +\nsv_\nw)$.
\end{itemize}

\new{The proof of property (ii) is now quite simple.}
For example, observe that \new{$\sigma_\scl\bw$} 
is simply distributed as $\Dexp_{\R_1^n,\sigma_\scl}$. 
Using Fact~\ref{app:fct:bound-on-collaped-distribution} in the Appendix, 
$\bw$ is pointwise close to a (continuous) Gaussian. 
Therefore, $\bw^{\perp\bs}$ and $\bw^\bs$ are nearly independent. 
Finally, observe that rejecting sampling based on $y$ 
does not change the conditional distribution $\bw^{\perp \bs}\mid\bw^\bs=l$. 
This suffices to show property (ii).

As for property (i), we first note that the rejection sampling 
ensures that $k$ (the parameter in Step~\ref{lne:first-rejection} of Algorithm~\ref{algo:sampling}) 
is uniform over $B$. Recall also that $\bw$ is pointwise close 
to a continuous Gaussian; this implies the same for $\bw^{\bs}$. 
In fact, $\bw^{\bs}$ is pointwise close to $\Dgaus_1$. 
Furthermore, when we condition on a particular fixed $k$ 
(or equivalently fixed $y$), the noise term $\nsv_\nw$ 
can only take values $i - \sigma_\scl\|\bs\|_2\bw^{\bs}$, 
for some $i \in \Z + \frac{k}{t + k - \psi}$. 
Roughly speaking, we may view \newblue{$\sigma_\scl\|\bs\|_2\bw^{\bs}$} 
as our ``signal'' and 
\newblue{$i-\sigma_\scl\|\bs\|_2\bw^{\bs}$} as our ``noise''. 
The \newblue{signal ratio ($\sn/(\sn+\ns)$) is thus 
$\sr=\frac{(\sigma_\scl\sqrt{n})^2}{(\sigma_\scl\sqrt{n})^2 + (\sigma^2+\|\mathbf{s}\|_2^2\sigma_\ad^2)}$}, 
where the equality follows from $\|\bs\|_2^2 = n$ 
and our choice of $\sigma_\ad$. 

\paragraph{Null Hypothesis Analysis}
For $D_{t,\eps,\psi,B,\delta}^\nul$, 
we can show that it is pointwise close to $\Dgaus_{\R^n,1}$:

\begin{lemma} \label{lem:rejection-null-distribution}
For any $\bu\in \R^n$, we have that 
$P_{D_{t,\epsilon,\psi,B,\delta}^\nul}(\bu)= 
(1\pm O(\delta))P_{\Dgaus_{\R^n,1}}(\bu)\; .$
\end{lemma}

The proof of this case is relatively simple: 
$\bx, y$ are independent, so even after conditioning 
on a specific value of $y$, $\bx$ remains uniformly distributed over $\R_1^n$. 
This means that $\mod_1(\bx + \bx_\ad)$ is uniformly random 
over $\R_1^n$ and the distribution 
$\Dgaus_{\bx+\bx_\ad+\Z^n,\sigma_\scl'}$ simplifies to 
$D^{\mathrm{expand}}_{\R_1^n,\sigma_\scl'}$. 
At this point, an application of Fact \ref{app:fct:bound-on-collaped-distribution} 
in the appendix concludes the proof. 

\subsubsection{The Reduction Algorithm} \label{subsubsec:reduction-algorithm}

With the rejection sampling algorithm (Algorithm \ref{algo:sampling}) at our disposal, 
we can now give the full construction of the hard instance. 
We use $D_{t,\eps,\psi_+,B_+,\delta}$ for $\bx \mid y=+1$, 
$D_{t,\epsilon,\psi_-,B_-,\delta}$ for $\bx \mid y=-1$ 
(with a carefully chosen pair of $(B_+,\psi_+)$ and $(B_-,\psi_-)$, 
as we discussed in Section~\ref{subsec:carving-explaination}), 
and take a proper marginal distribution of $y$ 
to build a joint distribution of $(\bx,y)$.
We introduce a reduction algorithm that, 
\new{given samples from our LWE problem 
(either from the null or the alternative hypothesis),} 
produces i.i.d.\ samples $(\bx,y)$ from a
joint distribution with the following properties:
\begin{enumerate}[leftmargin=*]
	\item If the input LWE problem is the null hypothesis, then $\bx \mid y=+1$ and $\bx \mid y=-1$ 
	are close in total variation distance. 
	Therefore, no hypothesis for predicting $y$ in terms of $\bx$ 
	can do much better than the best constant hypothesis.
	\item If the input LWE problem is the alternative hypothesis, 
	then the joint distribution of $(\bx,y)$ we build is close to a distribution $D$ 
	that satisfies $O(\eta)$ Massart condition with respect to a 
	degree-$O(t/\eps)$ PTF, and there is a degree-$O(t/\eps)$ PTF with small error on $D$. 
\end{enumerate}
We formalize the idea from Section~\ref{subsec:carving-explaination} here.
For $\bx \mid y=+1$, we will use $\psi_+\eqdef 0$ and $B_+\eqdef [0,\eps]$; 
this is equivalent to the \cite{diakonikolas2021near} construction. 
For $\bx \mid y=-1$, we take $\psi_-\eqdef t/2$, which is also the same as \cite{diakonikolas2021near};
but \newblue{instead of taking $B_-\eqdef[t/2,t/2+\eps]$,} 
we will need to carve out the slots on $B_-$.
First, we define the mapping 
\newblue{
$g:\R-[-1.5t,0.5t]\mapsto [0.5t, t]$, as follows: for $i\in \Z$ and $b\in \R_t$, 
we have that
\begin{align*}
    g(it+t/2+b)\eqdef
    \begin{cases}
        \frac{b}{i+1}+t/2 &\text{if $i\geq 0$;}\\
        \frac{b-t}{i+2}+t/2 &\text{if $i<0$.}
    \end{cases}
\end{align*}}
This function maps a location \new{$it+t/2+b$} 
to the corresponding place we need to carve out on $B_-$, 
which is defined in Algorithm~\ref{algo:reduction}.
These intervals are chosen so that the decaying density 
part of $+1$ can fit in, as we discussed in 
Section~\ref{subsec:carving-explaination}. 
Now we introduce the algorithm that reduces LWE to learning Massart PTFs.

\begin{algorithm} [h]
\caption{Reducing LWE to Learning PTFs with Massart Noise} \label{algo:reduction}
\textbf{Inputs:} 
$m$ samples from an instance of 
$\lwe(m,\R_1^n,\{\pm 1\}^n,\mathcal{N}_{\newblue{\sigma}},\mod_1)$.
The input parameters are
$m' \in \N, t, \epsilon \in\R_{>0}$, $\delta\in (0,1)$, and $\eta>0$ a sufficiently small value.
In addition, \new{the parameters satisfy Condition~\ref{cond:params}.}

\textbf{Output:} $m'$ many samples in $\R^n\times \{\pm 1\}$ or FAIL.

\begin{enumerate}[leftmargin=*]
	\item We take $\psi_+=0$, $B_+=[0,\eps]$, $\psi_-=t/2$ and 
    \begin{align*}
	    B_-\eqdef &[t/2,t/2+\epsilon] \\
	    &-\bigcup\limits_{i={\frac{t}{2\eps}-1}}^{\frac{t}{\eps}-1}g([it-2c'\epsilon,it])
	    -\bigcup\limits_{i={\frac{t}{2\eps}-1}}^{\frac{t}{\eps}-1} 
	    g([it+(i+1)\epsilon,it+(i+1)\epsilon+2c'\epsilon])\\
	    &-\bigcup\limits_{i={-\frac{t}{\eps}}-1}^{-\frac{t}{2\eps}-1} 
	    g([it+(i+1)\epsilon-2c'\epsilon,it+(i+1)\epsilon])
	    -\bigcup\limits_{i={-\frac{t}{\eps}}-1}^{-\frac{t}{2\eps}-1} 
	    g([it,it+2c'\epsilon])\; .
    \end{align*}

	\item
	Repeat the following $m'$ times. If at any point the algorithm 
	attempts to use more than $m$ LWE samples from the input, then output FAIL.
	\begin{enumerate}
	    \item With probability $1 - \eta$, repeat the following until Algorithm~\ref{algo:sampling} 
	    accepts and output $\bx'$: run Algorithm~\ref{algo:sampling} with the next unused LWE sample 
	    from the input and parameters $t, \eps, \psi = \psi_+, B = B_+, \delta$. Add $(\bx', +1)$ to the output samples. 
	    \item With probability $\eta$, repeat the following until Algorithm~\ref{algo:sampling} accepts 
	    and output $\bx'$: run Algorithm~\ref{algo:sampling} with the next unused LWE sample 
	    from the input  and parameters $t, \eps, \psi = \psi_-, B = B_-, \delta$. 
	    Add $(\bx', -1)$ to the output samples. 
	\end{enumerate}
\end{enumerate}
\end{algorithm}

\smallskip

We similarly define the output distributions of the algorithms in the two cases as follows:

\begin{definition}
Let $D_\PTF^\alt$ be mixture of $D_{t,\eps,\psi_+,B_+,\delta}^\alt$ 
and $D_{t,\eps,\psi_-,B_-,\delta}^\alt$ with $+1$ and $-1$ labels 
and weights $1 - \eta$ and $\eta$ respectively. Similarly, let $D_\PTF^\nul$ 
be mixture of $D_{t,\eps,\psi_+,B_+,\delta}^\nul$ and 
$D_{t,\eps,\psi_-,B_-,\delta}^\nul$ with $+1$ and $-1$ labels 
and weights $1 - \eta$ and $\eta$ respectively.
\end{definition}

The following observation is immediate from the algorithm.

\begin{observation} \label{obv:reduction-algorithm-iid}
In the alternative (resp. null) hypothesis case,
the output distribution of Algorithm \ref{algo:reduction}, conditioned on not failing, 
is the same as $m'$ i.i.d.\ samples drawn from $D_\PTF^\alt$ (resp. $D_\PTF^\nul$).
\end{observation}

\paragraph{Alternative Hypothesis Analysis}
We prove that there exists a degree-$O(t/\eps)$ PTF 
such that $D_\PTF^\alt$ is close to (in total variation distance) satisfying the 
$O(\eta)$ Massart noise condition with respect to this PTF, 
and this PTF has small error with respect to $D_\PTF^\alt$.

\begin{lemma} \label{lem:ptf-alt-distr}
\newblue{
$D_\PTF^\alt$ is $O(\delta/m')$ close in total variation distance to a distribution $D^\tru$ such that 
there is a degree-$O(t/\eps)$ PTF $\sgn(p(\bx))$ that:}
\begin{enumerate}[leftmargin=*]
    \item \newblue{$\E_{(\bx,y)\sim D^\tru}[\sgn(p(\bx))\neq y]\leq\exp(-\Omega(t^4/\eps^2))$; and}
    \item $D^\tru$ satisfies the $O(\eta)$ Massart noise condition 
    with respect to $\sgn(p(\bx))$.
\end{enumerate}
\end{lemma}

To prove this lemma, we extend the proof of~\cite{diakonikolas2021near} 
to be able to handle the noise from the previous step. 
We start by truncating the noise so that it belongs to $[-c' \eps, c' \eps]$. 
Due to the last condition of parameters in 
\newblue{Condition~\ref{cond:params}} for Algorithm \ref{algo:reduction}, 
the probability mass truncated is at most $O(\delta/m')$, 
which corresponds to these terms in the above lemma. 
Once we have considered the truncated version of the noise, 
our carving ensures that, at any point in the support of $+1$, 
the probability mass of $+1$ is at least $\Omega(1)$ times 
that of $-1$ (and vice versa). 
The remainder of the proof then works similarly 
to~\cite{diakonikolas2021near}.

\newblue{
\paragraph {Null Hypothesis Analysis}
Here we analyze the null hypothesis case
by showing that for $(\bx,y)\sim D_{\PTF}^\nul$, 
$\bx$ and $y$ are almost independent.
For $D_\PTF^\nul$, we establish the following property.

\begin{lemma} \label{app:lem:ptf-nul-distr}
For any $\bu\in \R^n$, we have that
\begin{align*}
\pr_{(\bx,y)\sim D_\PTF^\nul}[y=+1\mid \bx=\bu]=(1\pm O(\delta))(1-\eta)\;.
\end{align*}
\end{lemma}

\noindent This lemma follows directly from Lemma \ref{lem:rejection-null-distribution}, since 
$$P_{\bx|(y=+1)}(\bu), P_{\bx|(y=-1)}(\bu)
=(1\pm O(\delta))P_{D_{\R^n,1}^\gaus}(\bu)
\; .$$
Also the marginal probability is $\pr[y=+1]=1-\eta$.
Therefore, for any $\bu\in \R^n$, 
$$\pr_{(\bx,y)\sim D_\PTF^\nul}[y=+1\mid \bx=\bu]=(1\pm O(\delta))(1-\eta)\; .$$
}

\paragraph{Putting Everything Together}
Having reduced LWE to learning Massart PTFs, 
we can apply a Veronese mapping on the samples; 
this PTF becomes an LTF on the Veronese mapping. 
Since we use degree-$O(t/\eps)$ Veronese mapping, 
the dimension for the Massart LTF problem is
$N = (n + 1)^{O(t/\eps)}$. The full proof of 
Theorem~\ref{thm:lwe-to-massart} is deferred 
to Appendix \ref{app:main-reduction}.

\section{Discussion}
\label{sec:discussion}

\newblue{
In this paper, we gave the first super-polynomial computational hardness result
for the fundamental problem of PAC learning halfspaces with Massart noise.
Prior to our work, no such hardness result was known, even for exact learning.
Moreover, our hardness result is quantitatively essentially tight,
nearly-matching the error guarantees of known polynomial time algorithms for the problem.

More concretely, our result rules out the existence of polynomial
time algorithms achieving error smaller than $\Omega(\eta)$, where $\eta$
is the upper bound on the noise rate, even of the optimal error is very small,
assuming the subexponential time hardness of LWE.
A technical open question is whether the constant factor
in the $\Omega(\eta)$-term of our lower bound
can be improved to the value $C=1$; this would match known algorithms exactly.
(As mentioned in the introduction, such a sharp lower bound
has been recently established in the SQ model~\cite{nassert22}, 
improving on~\cite{diakonikolas2021near}.)

It is also worth noting that our reduction rules out polynomial-time algorithms, but
does not rule out, e.g., subexponential or even quasipolynomial time algorithms
with improved error guarantees. We believe that obtaining stronger hardness 
for these problems would require substantially new ideas, 
as our runtime lower bounds are essentially the same as the best time lower bounds 
for learning in the (much stronger) agnostic noise model 
or in restricted models of computation (like SQ). This seems related 
to the requirement that our bounds require subexponential hardness of LWE in our assumption. 
As the strongest possible assumptions only allow us to prove quasi-polynomial lower bounds, 
any substantially weaker assumption will likely fail to prove super-polynomial ones.

}

\bibliographystyle{alpha}
\bibliography{allrefs}

\newpage

\newpage

\appendix

\section*{Appendix}

\paragraph{\bf Additional Notation} 
For $d\in \N$, we use $V_d:\R^n\mapsto \R^{{(n+1)}^d}$ 
to denote the degree-$d$ Veronese mapping, 
where the outputs are corresponding to monomials of degrees at most $d$.

\section{Additional Technical Background}
\label{app:background}
For completeness, we start with the definition of the classic LWE problem. 
Note that Definition \ref{def:general-lwe-one-dim} generalizes this definition,
so we will not use the definition below directly.

\begin{definition} [Classic Learning with Errors Problem] \label{app:def:classic-lwe}
Let $m,n,q \in \N$, and let 
$D_\mathrm{sample}$, $D_\mathrm{secret}$, $D_\mathrm{noise}$ 
be distributions on $\Z_q^n$, $\Z_q^n$, $\Z_q$ respectively. 
In the $\lwe_\cl(m,D_\mathrm{sample},D_\mathrm{secret},D_\mathrm{noise},\mathrm{mod}_q)$ problem, 
we are given $m$ independent samples $(\mathbf{x},y)$ 
and want to distinguish between the following two cases:
\begin{enumerate}[leftmargin=*]
	\item [(i)] {\bf Alternative hypothesis}: $\bs$ is drawn from $D_\mathrm{secret}$. 
	Then, each sample is generated by taking $\bx\sim D_\mathrm{sample}, \nsv\sim D_\mathrm{noise}$
	and letting $y= \mathrm{mod}_q(\left<\bx, \bs\right>+\nsv)$.  

	\item [(ii)] {\bf Null hypothesis}:  
	$\bx, y$ are independent and each has the same marginal distribution as above.
\end{enumerate}
\end{definition}

Throughout our proofs, we need to manipulate discrete Gaussian distributions that 
are taken modulo 1 and those with noise added. 
Due to this, it will be convenient to introduce the following definitions.

\begin{definition} [Expanded Gaussian Distribution from $\R_1^n$]
For $\sigma \in \R_{+}$,  let $\Dexp_{\R_1^n,\sigma}$ 
denote the distribution of $\bx'$ drawn as follows: 
first sample $\bx \sim U(\R_1^n)$ (using the Lebesgue measure on $\R_1^n$), 
and then sample $\bx' \sim \Dgaus_{\Z^n+\bx,\sigma}$.
\end{definition}

\begin{definition} [Collapsed Gaussian Distribution on $\R_1^n$]
We will use $\Dcol_{\R_1^n,\sigma}$ to denote the distribution of 
$\mathrm{mod}_1(\bx)$ on $\R_1^n$, 
where $\bx\sim \Dgaus_{\R^n,\sigma}$. 
\end{definition}

We will also need the following fact, which can be 
easily derived from known bounds in literature. 
For completeness, we provide the full proof in Appendix~\ref{app:collapsed-bound}.

\begin{fact} \label{app:fct:bound-on-collaped-distribution}
Let $n \in \N, \sigma \in \R_+, \eps \in (0, 1/3)$ be such that $\sigma \geq \sqrt{\ln(2n(1 + 1/\eps)) / \pi}$. 
Then, we have
$$\frac{P_{\Dexp_{\R_1^n,\sigma}/\sigma}(\bt)}{P_{\Dgaus_{\R^n,1}}(\bt)} 
=  \frac{P_{U(\R_1^n)}(\mathrm{mod}_1(\sigma \bt))}{P_{\Dcol_{\R_1^n,\sigma}}(\mathrm{mod}_1(\sigma \bt))} 
= 1\pm O\left (\eps\right) \;,$$
for all $\bt \in \R^n$,
and
$$d_{\mathrm{TV}}\left (\frac{\Dexp_{\R_1^n,\sigma}}{\sigma},\Dgaus_{\R^n,1}\right ), 
d_{\mathrm{TV}}\left (\Dcol_{\R_1^n,\sigma},U(\R_1^n)\right )=
\exp\left (-\Omega({\sigma^2})\right ) \;.$$
\end{fact}

We will also need the following well-known fact in our proof.

\begin{lemma} [Corollary 3.10 of \cite{regev2009lattices}] \label{app:lem:discrete-gaussian-to-gaussian}
Let $\mathbf{z}\in \R^n$, $\sigma_1,\sigma_2\in \R_{>0}$. 
Assume that 
$$1/\sqrt{1/\sigma_1^2+\|\mathbf{z}\|_2^2/\sigma_2^2}\geq \eta_\epsilon(L) \;,$$
and further suppose that $\mathbf{x}\sim \Dgaus_{\Z^n ,\sigma_1}$ and $\mathbf{x}'\sim \Dgaus_{\sigma_2}$.
Then the distribution of $\langle\mathbf{x},\mathbf{z}\rangle +\mathbf{x}'$ is within 
$O(\epsilon)$ total variation distance to $\Dgaus_{\sqrt{\|\mathbf{z}\|_2^2\sigma_1^2+\sigma_2^2}}$.
($\eta_\epsilon(L)$ is the smoothing parameter of a lattice, 
and is defined in Definition~\ref{app:def:smoothing-parameter}.)
\end{lemma}

\subsection{Proof of Fact \ref{app:fct:bound-on-collaped-distribution}} \label{app:collapsed-bound}
We start by recalling the definition of a lattice.

\begin{definition}[Lattice] \label{app:def:lattice}
Let $\mathcal{B} = (\bv_1,\bv_2,\cdots ,\bv_n)$ be a set of  $n$ linearly independent vectors in $\R^n$.
The lattice $L = L(\mathcal{B})$ defined by $\mathcal{B}$ is the set of all integer linear combinations
of vectors in $\mathcal{B}$, i.e., the set 
$\{\bv \in \R^n: \bv = \sum_{j=1}^n \alpha_j \bv_j, \alpha_j \in \Z \}$.
\end{definition}

Since we only use the integer lattices $\Z^n$, we will only introduce notation as necessary. 
For a more detailed introduction about lattices, 
the reader is referred to \cite{micciancio2002complexity}.

Partially supported Gaussian distributions (Definition~\ref{def:dpart}) 
behave similarly to continuous Gaussian distributions. 
The similarity can be quantified based on the so-called 
smoothing parameter of a lattice defined below.

\begin{definition} [see, e.g., Definition 2.10 of \cite{regev2009lattices}] \label{app:def:smoothing-parameter}
For an $n$-dimensional lattice $L$ and $\eps \in \R_+$, 
we define the \emph{smoothing parameter} $\eta_\epsilon(L)$ 
to be the smallest $s$ such that\footnote{Note that $L^*$ denotes the \emph{dual lattice} of $L$; 
it contains all $\by$ such that $\left<\bx, \by\right>$ is an integer for all $\bx \in L$.} 
$\rho_{1/s}(L^*\setminus \{0\})\leq \epsilon$.
\end{definition}

\begin{lemma}[see, e.g., Lemma 2.12 of \cite{regev2009lattices}] \label{app:lem:smoothing-integer-lattice}
For any $n \in \N$ and $\eps \in \R_+$, 
we have that $\eta_\eps(\Z^n)\leq\sqrt{\frac{\ln(2n(1+1/\eps))}{\pi}}\;.$
\end{lemma}

The main lemma we will use here is that, 
when $\sigma$ is larger than the smoothing parameter, 
the normalizing factor remains roughly the same 
after a shift by an arbitrary vector $\bv$, as formalized below. 
(Note that in the discrete Gaussian case, the two sides would have been equal.) 
This property follows from the proof of \cite[Lemma 4.4]{micciancio2007worst}, 
in which it was shown that 
$\rho_{\sigma}(L + \bv) \in [(1 - \eps)\det(L^*), (1 + \eps)\det(L^*)]$;
the lemma then follows since $\det((\Z^n)^*) = 1$.

\begin{lemma} [\cite{micciancio2007worst}]\label{app:lem:general-pointwise-bound}
Let $n \in \N, \eps \in (0, 1)$ and $\sigma\geq \eta_\eps(\Z^n)$. 
Then, for any $\bv \in \R^n$,  we have that 
$\rho_{\sigma}(\Z^n + \bv) \in [1 - \eps, 1 + \eps].$	
\end{lemma}

We are now ready to prove Fact~\ref{app:fct:bound-on-collaped-distribution}.

\begin{proof}[Proof of Fact~\ref{app:fct:bound-on-collaped-distribution}]
Let $\br = \mathrm{mod}_1(\sigma \bt)$.
Notice that
\begin{align*}
\frac{P_{(1/\sigma)\circ\Dexp_{\R_1^n,\sigma} }(\bt)}{P_{\Dgaus_{\R^n,1}}(\bt)}
&= \frac{1}{P_{\Dgaus_{\R^n,1}}(\bt)} \cdot P_{U(\R_1^n)}(\br) \cdot P_{\Dgaus_{\Z^n + \br, \sigma}}(\sigma \bt) \\
&= \frac{1}{P_{\Dgaus_{\R^n,1}}(\bt)} \cdot P_{\Dgaus_{\Z^n + \br, \sigma}}(\sigma \bt) \\
&= \frac{1}{P_{\Dgaus_{\R^n,1}}(\bt)} \cdot \frac{\rho_{\sigma}(\sigma\bt)}{\rho_{\sigma}(\Z^n + \br)} \\
&= \frac{1}{P_{\Dgaus_{\R^n,1}}(\bt)} \cdot \frac{p_{1}(\bt)}{\rho_{\sigma}(\Z^n + \br)} \\
&= \frac{1}{\rho_{\sigma}(\Z^n + \br)} \\
&\overset{(\text{Lemmas~\ref{app:lem:general-pointwise-bound},~\ref{app:lem:smoothing-integer-lattice}})}{=} \frac{1}{1 \pm \eps} \\
&= 1 \pm O(\eps) \;.
\end{align*}
Notice also that 
\begin{align*}
\frac{P_{\Dcol_{\R_1^n,\sigma}}(\br)}{P_{U(\R_1^n)}(\br)} &= P_{\Dcol_{\R_1^n,\sigma}}(\br) \\
&= \sum_{\bu \in \Z^n} \rho_{\sigma}(\bu + \br) \\
&= \rho_{\sigma}(\Z^n + \br) \\
&= \frac{P_{\Dgaus_{\R^n,1}}(\bt)}{P_{(1/\sigma)\circ\Dexp_{\R_1^n,\sigma}}(\bt)} \;,
\end{align*}
where the last equality follows from the previously derived equality above.
These prove the first equality in Fact~\ref{app:fct:bound-on-collaped-distribution}.

Next, note that pointwise closeness immediately yields the same bound on the TV distance. 
Therefore, we have 
$$d_{\mathrm{TV}}\left ((1/\sigma)\circ \Dexp_{\R_1^n,\sigma},\Dgaus_{\R^n,1}\right ),
d_{\mathrm{TV}}\left (\Dcol_{\R_1^n,\sigma},U(\R_1^n)\right ) \leq \eps \leq \exp(-\Omega(\sigma^2)) \;,$$
where the second inequality follows from our choice of $\sigma$.
\end{proof}

\section{Hardness of LWE with Binary Secret, Continuous Samples, and Continuous Noise}
\label{app:lwe-reduction}

The main theorem of this section is a reduction from 
$\lwe(n,\Z_q^l,\Z_q^l,\Dgaus_{\Z,\sigma},\mod_q)$ to \\ 
$\lwe(m,\R_1^n,\{\pm 1\}^n,\Dgaus_{\sigma'},\mod_1)$. 
The purpose of this reduction is to massage the LWE problem 
into its variant with binary secret, continuous samples and continuous noise,
so we can further reduce it to the $\massart$ problem.

We reiterate that this reduction between LWE problems 
follows from the previous works~\cite{Mic18,vinod2022}; 
we provide the full proof here for completeness. 
The main theorem regarding the reduction is presented below.

\begin{theorem} \label{app:thm:main-lwe-reduction}
Let $n,m,l,q\in \N$, $\sigma,\sigma'\in \R$, where the parameters satisfy:
\new{
\begin{enumerate}
    \item $\log (q)/2^l=\delta^{\omega(1)}$ \newblue{(where $\omega(1)$ goes to infinity as $n$ goes to infinity)},
    \item $\sigma= \omega(\sqrt{\log(mn/\delta)})$ 
    \item $n\geq 2l\log_2(q)+\omega(\log(1/\delta))$, and
	\item $\sigma'=c\sqrt{n}\, \sigma/q$, where $c$ is a sufficiently large constant.
\end{enumerate}} 
\noindent Suppose there is no $T+\poly(m,n,l,\log q,\log(1/\delta))$ time distinguisher for 
$\lwe(n,\Z_q^l,\Z_q^l,\Dgaus_{\Z,\sigma},\mod_q)$ 
with $ \eps/m$ advantage. 
Then there is no $T$ time distinguisher for 
$\lwe(m,\R_1^n,\{\pm 1\}^n,\gaus_{\sigma'},\mod_1)$ 
with $2\eps+O(\delta)$ advantage. 
\end{theorem}

Before we prove Theorem \ref{app:thm:main-lwe-reduction}, 
let us note that combining it with Assumption \ref{asm:LWE-hardness} 
yields Lemma~\ref{lem:continuous-lwe-hardness}.

\begin{proof}[Proof of Lemma~\ref{lem:continuous-lwe-hardness}]
We take 
$n=l^\alpha$, $m=2^{O(l^{\beta'})}$, $q=l^{\gamma'}$,
and $\sigma=\sqrt{l}$, where $\alpha>1$ and $\beta''>\beta'\in (0,1)$.
Then, from Assumption \ref{asm:LWE-hardness}, it follows that
there is no $2^{O(n^{\beta''/\alpha})}$ time algorithm 
to solve $\lwe\left (n,\Z_q^l,\Z_q^l,\Dgaus_{\Z, n^{\alpha/2}},\mod_q\right)$ 
with $2^{-O(n^{\beta''/\alpha} )}$ advantage.

We take $\delta$ to be a sufficiently small constant 
and apply \newblue{Theorem}~\ref{app:thm:main-lwe-reduction}.
Then we have that no $2^{O(n^{\beta'/\alpha})}$ time algorithm can solve 
$\lwe\left(2^{O(n^{\beta'/\alpha})},\R_1^n,\{\pm 1\}^n,\Dgaus_{\Z,O\left (n^{1/(2\alpha) 
+1/2-\gamma'/\alpha}\right )},\mod_1\right )$ with $1/3$ advantage.
We rename $\beta=\beta'/\alpha$ and $\gamma=1/(2\alpha)+1/2-\gamma'/\alpha$.
By taking $\alpha>1$ and $\beta'<1$ to be arbitrarily close to $1$ and 
$\gamma'$ to be arbitrarily large constant, 
we can have $\beta\in (0,1)$ to be arbitrarily close to $1$ 
and $\gamma$ to be arbitrarily large.  
Then there is no $2^{O(n^\beta)}$ time algorithm to solve 
$\lwe\left(2^{O(n^\beta)},\R_1^n,\{\pm 1\}^n,\Dgaus_{O\left (n^{-\gamma}\right )}, \mod_1\right)$ 
with $1/3$ advantage.

Given the above, by a standard boosting argument, 
it follows that there is no $2^{O(n^\beta)}$ time algorithm to solve 
$\lwe\left (2^{O(n^\beta)},\R_1^n,\{\pm 1\}^n,\Dgaus_{O\left (n^{-\gamma}\right )},\mod_1\right )$ 
with $2^{-O(n^\beta)}$ advantage.
\end{proof}

\paragraph{From $\Z_q$ Secret to Binary Secret}
\new{
We require the following lemma from \cite{vinod2022} that reduces the
classic LWE problem with $\Z_q$ secret to an LWE problem with binary secret.

\begin{lemma} [Theorem 7 from \cite{vinod2022}] \label{app:lem:lwe-binary-secret}
Let $q,l,n,m\in \N$ and $\sigma\in \R_+$. 
Assuming that there is no time $T+\poly(l,n,\log(q),\log(1/\delta))$ algorithm for solving 
$\lwe(n+1,\Z_q^l,\Z_q^l,\Dgaus_{\Z,\sigma},\mod_q)$ with  advantage $(\eps-\delta^{\omega_n(1)})/(2m)$,
there is no time $T$ algorithm for solving
$\lwe(m,\Z_q^{n+1},\{\pm 1\}^{n+1},\Dgaus_{\Z,\sigma'},\mod_q)$ with $\eps$ advantage,
as long as the following holds:
$\log (q)/2^l=\delta^{\omega_n(1)}$,
$\sigma\geq 4\sqrt{\omega(\log(1/\delta))+\log m+\log n}$,
$n\geq 2l\log_2(q)+\omega(\log(1/\delta))$, and
$\sigma'=2\sigma\sqrt{n+1}$.
\end{lemma}
}

\paragraph{From Discrete To Continuous}

In the next step, we show that adding a small amount 
of Gaussian noise on $y$ will render the discrete Gaussian noise 
close to continuous Gaussian noise.

\begin{lemma} [Lemma 15 from \cite{vinod2022}] \label{app:lem:lwe-continuous-noise}
Let $n,m,q\in \N$, $\sigma\in \R_+$, $c$ be a sufficiently large constant and 
suppose that $\sigma>\sqrt{c\log (m/\delta)}$.  
Suppose there is no distinguisher for $\lwe(m,\Z_q^n,\{\pm 1\}^n,\Dgaus_{\Z,\sigma},\mod_q)$
running in time \mbox{$T+\poly(m,n\log(q),\log(1/\delta))$ with $\eps$} advantage. 
Then there is no $T$-time distinguisher for \mbox{$\lwe(m,\Z_q^n,\{\pm 1\}^n,\Dgaus_{\sigma'},\mod_q)$} 
with $\eps+O(\delta)$ advantage, where 
$$\sigma'\geq\sqrt{\sigma^2+c\log (m/\delta)}\; .$$
\end{lemma}
\begin{proof} 
We will give a reduction argument. 
Take $\sigma_{\mathrm{add}}=\sqrt{\sigma'^2-\sigma^2}$. 
Then for each sample $(\bx,y)$ 
from $\lwe(m,\Z_q^n,\{\pm 1\}^n,\Dgaus_{\Z,\sigma},\mathrm{mod}_q)$, 
we return $$(\bx,\mathrm{mod}_1(y+e))\text{ where }e\sim \Dgaus_{\sigma_{\mathrm{add}}}$$ 
as a sample for $\lwe(m,\Z_q^n,\{\pm 1\}^n,\Dgaus_{\sigma'},\mathrm{mod}_q)$. 

Suppose that the input instance is in the alternative hypothesis case. 
We need to argue that after running the reduction algorithm, 
the new noise $z+e$ has at most $O(\delta/m)$ total variation distance from $\Dgaus_{\sigma'}$. 
From Lemma~\ref{app:lem:discrete-gaussian-to-gaussian}, 
we have that for a sufficiently large constant $c$, 
$(z+e)$ is within $O(\delta/m)$ total variation distance to $\Dgaus_{\sigma'}$. 
With $m$ samples, this only decreases the distinguishing advantage 
by at most $O(\delta)$. 
	
Suppose that input instance is from the null hypothesis case. 
We need to show after the reduction algorithm, 
both $\bx$ and $y$ have the same marginal as in the previous case. 
It is easy to verify that after the reduction, 
$\bx$ will have the same marginal as in the previous case. 
For $y$, the marginal distribution of $y$ in the previous case is $U(\R_q)$ by symmetry. 
Similarly, in the null hypothesis case, the distribution of 
$y$ is also $U(\R_q)$ by symmetry.
\end{proof}

We now show that adding a small amount of Gaussian noise 
on the samples will render the samples continuous; 
at the same time, this extra Gaussian noise 
on samples can be interpreted 
as some extra Gaussian noise on the labels. 

\begin{lemma} [Lemma 16 from \cite{vinod2022}] \label{app:lem:lwe-continuous-sample}
Let $n,m,q\in \N$, $\sigma\in \R$, $c$ be a sufficiently large constant. Suppose that 
$\sigma\geq cn^{1/2}\sqrt{\log(mn/\delta)}$. 
Suppose there is no $T+\poly(m,n,\log(q),\log(1/\delta))$-time distinguisher for 
$\lwe(m,\Z_q^n,\{\pm 1\},\Dgaus_{\sigma},\mod_q)$ with $\eps$ advantage.
Then there is no $T$-time distinguisher for $\lwe(m,\R_q^n,\{\pm 1\},\Dgaus_{\sigma'},\mod_q)$
with $\eps+\delta$ advantage, where
$$\sigma'=\sqrt{\sigma^2+cn\log (mn/\delta)}\; .$$
\end{lemma}

\begin{proof} We will give a reduction argument. 
Taking $\sigma_{\mathrm{add}}=\sqrt{\frac{\sigma'^2-\sigma^2}{n}}$, 
then for each sample $(\bx,y)$ 
from $\lwe(m,\Z_q^n,\{\pm 1\}^n,\Dgaus_{\sigma},\mathrm{mod}_q)$, 
we return  	
$$(\mathrm{mod}_q(\mathbf{x+x'}),y)\text{ ,where }\mathbf{x'}\sim \Dgaus_{\R^n,\sigma_{\mathrm{add}}}$$ 
as a sample for $\lwe(m,\R_q^n,\{\pm 1\}^n,\Dgaus_{\sigma'},\mathrm{mod}_q)$.

Suppose the input instance is in the alternative hypothesis case. We need to show the following:
\begin{enumerate}
	\item [(a)] $\mathrm{mod}_q(\mathbf{x+x'})$ is close to $U(\R_q^n)$; and 
	
	\item [(b)] $y=\mathrm{mod}_q(\langle \mathrm{mod}_q(\mathbf{x+x'}),\bS\rangle+z')$, 
	where $z'$ is the noise in this new LWE instance we generated. 
	We need to show $z'$ has distribution close to a independent  $\Dgaus_{\sigma'}$ noise 
	(independent of $\mathrm{mod}_q(\mathbf{x+x'})$). 		
\end{enumerate}
For (a), from the symmetry of $\bx$ and $\bx+\mathbf{u}$ where $\mathbf{u}\in \Z_q^n$, we have
\begin{align*}
d_{\mathrm{TV}}(\mathrm{mod}_q(\mathbf{x+x'}),U(\R_q^n))&
=d_{\mathrm{TV}}(\mathrm{mod}_q(\mathbf{x+x'})|\mathrm{mod}_q(\bx+\mathbf{x'})\in [0,1]^n,U(\R_1^n))\\
&=d_{\mathrm{TV}}(\mathrm{mod}_1(\mathbf{x'}),U(\R_1^n))\;.
\end{align*}
From Fact \ref{app:fct:bound-on-collaped-distribution}, 
we know that for a sufficiently large constant $c$, 
the distribution of $\mathrm{mod}_q(\mathbf{x'})$ 
is $O(\delta/m)$ close to $U(\R_q^n)$.  

For (b), consider 
$\bx\sim U(\Z_q^n), s\sim U(\{\pm 1\}^n), z\sim \Dgaus_{\sigma}\text{ and }y
=\mathrm{mod}_q( \langle \bx,\bS\rangle+z)$. 
Then the new sample satisfies 
$$y=\mathrm{mod}_q(\langle \mathrm{mod}_q(\mathbf{x+x'}),\bS\rangle+(-\langle \mathbf{x'},\bS\rangle+z)) \;,$$ 
where the new noise is $z'=-\langle \mathbf{x'},\bS\rangle+z$.
We now verify the distribution of noise. 
Conditioned on a fixed $\bx+\mathbf{x'}$, 
we have noise as $-\langle \mathbf{x'},\bS\rangle+z$, 
where 
$\mathbf{x'}\sim D^\mathrm{partial}_{\bx+\mathbf{x'}+\Z^n, \sigma_\ad}$ and 
$z\sim \Dgaus_{\sigma}$. 
From Lemma \ref{app:lem:discrete-gaussian-to-gaussian}, 
we have that the distribution of noise is $O(\delta/m)$ 
close to $\Dgaus_{\sigma'}$, for sufficiently large $c$. 
Overall, with $m$ samples, the distinguishing advantage 
has decreased by at most $O(\delta)$.

If the input instance is from the null hypothesis case, 
it is easy to verify that after the reduction, 
both $\bx$ and $\mathbf{y}$ will have the same 
marginal as in the alternative hypothesis case.
\end{proof}

The final step is to rescale the sample and noise by $1/q$.

\begin{lemma} \label{app:rem:resacling}
Suppose there is no $T+\poly(m,n,\log q)$ time distinguisher for the distribution \\
$\lwe(m,\R_q^n,\{\pm 1\}^n,\Dgaus_{\sigma},\mathrm{mod}_q)$
with advantage $\eps$.  
Then there is no $T$ time distinguisher for the distribution 
$\lwe(m,\R_1^n,\{\pm 1\}^n,\Dgaus_{\sigma'},\mathrm{mod}_1)$ with advantage $\eps$, 
where $\sigma'=\sigma/q$. 
\end{lemma}

\begin{proof}
This follows simply by rescaling samples by $1/q$ 
and changing $\mathrm{mod}_q$ to $\mathrm{mod}_1$. 
Note the size of the secret remains unchanged here, 
but the noise is scaled by $1/q$. 
\end{proof}

\paragraph{Putting Things Together: Proof of Theorem~\ref{app:thm:main-lwe-reduction}}
\new{
Now we are ready to prove Theorem $\ref{app:thm:main-lwe-reduction}$.

\begin{proof} [Proof of Theorem $\ref{app:thm:main-lwe-reduction}$]
Suppose there is no $T+\poly(m,n,l,\log q,\log(1/\delta))$ time distinguisher for 
$\lwe(n,\Z_q^l,\Z_q^l,\Dgaus_{\Z,\sigma},\mod_q)$ with $\eps/m$ advantage.
Then, by applying Lemma \ref{app:lem:lwe-binary-secret}, we have that there is no 
$T+\poly(m,n,\log q,\log(1/\delta))$ time distinguisher for 
$\lwe(n,\Z_q^l,\{\pm 1\}^l,\Dgaus_{\Z,\sigma_1},\mod_q)$
with $2\eps+\delta^{\omega_n(1)}$ advantage, 
where $\sigma_1=2\sigma\sqrt{n+1}$.

Then, we apply Lemma \ref{app:lem:lwe-continuous-noise}, 
Lemma \ref{app:lem:lwe-continuous-sample}, and Corollary \ref{app:rem:resacling}.
It follows that there is no time $T$ distinguisher for 
$\lwe(m,\R_1^n,\{\pm 1\}^n,\gaus_{\sigma'},\mod_1)$ 
with $2\eps+O(\delta)$ advantage, 
where $\sigma'\geq \sqrt{\sigma_1^2+cn\log(mn/\delta)}/q$.

Recalling that $\sigma=\omega(\sqrt{\log (mn/\delta)})$, $\sigma_1=2\sigma\sqrt{n+1}$ 
and $\sigma'\geq \sqrt{\sigma_1^2+cn\log(mn/\delta)}/q$,
we have that $\sigma'=c\sqrt{n}\sigma/q$ is sufficient.
\end{proof}
}

\section{Omitted Proofs from Section \ref{sec:main-reduction}}
\label{app:main-reduction}

This section includes additional details of our Massart halfspace hardness reduction, 
and contains the proofs omitted from Section \ref{sec:main-reduction}.

We start with Figure \ref{fig:flow}, which shows a rough flow chart of the reduction algorithm 
and its relation with the relevant theorems and lemmas. 
\begin{figure}[H] 
	\includegraphics[width=16.4cm]{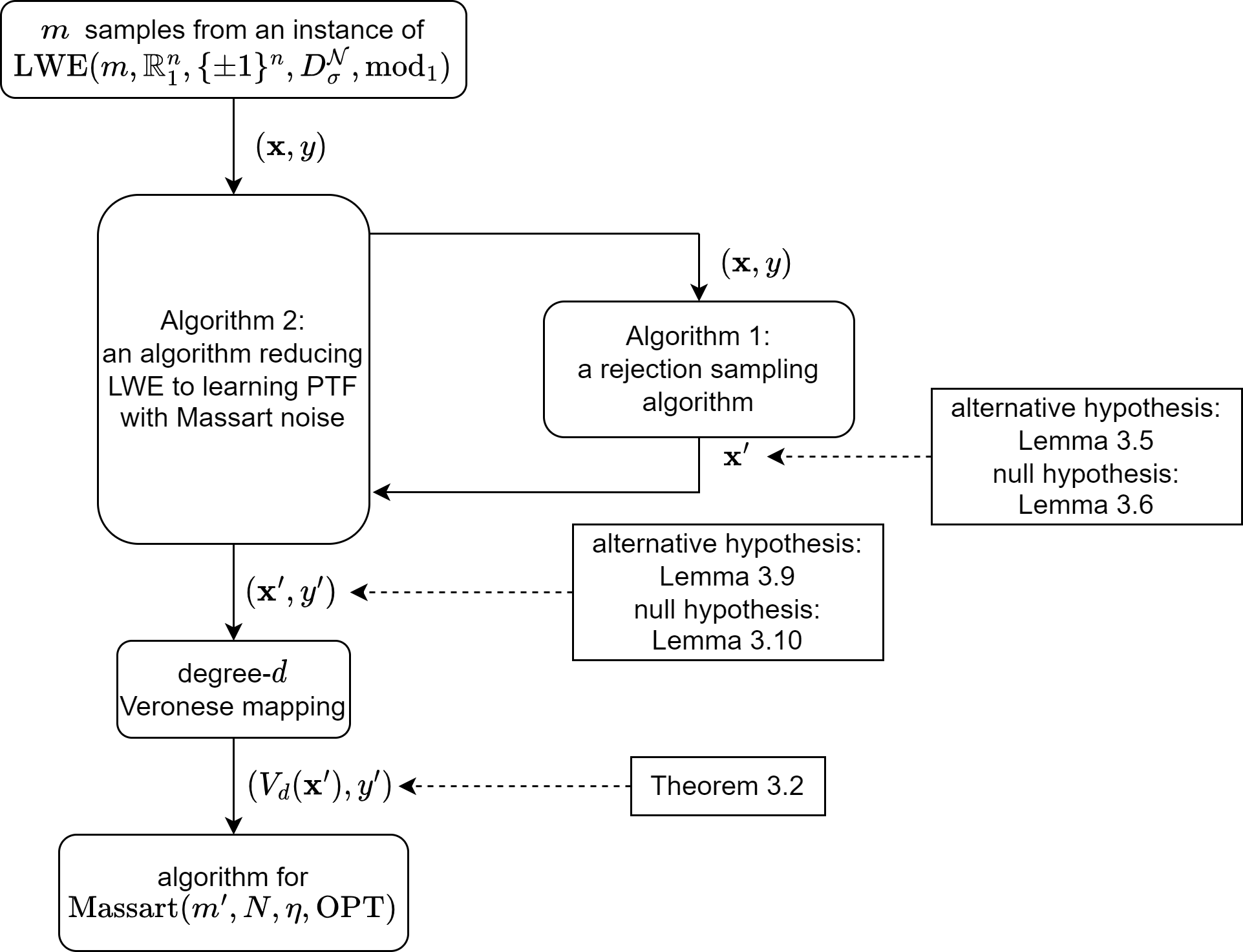}
	\caption{\textbf{Reducing LWE to Learning Halfspaces with Massart Noise.} 
	The diagram shows which step of the analysis each lemma is used for and 
	which case of the input LWE instance (alternative/null hypothesis case) for the reduction algorithm. 
	Lemma \ref{lem:rejection-alternative-distribution} (resp. Lemma \ref{lem:rejection-null-distribution}) 
	analyzes the properties of $\bx'$ when the input LWE instance is 
	from the alternative hypothesis case (resp. null hypothesis case).
	Lemma \ref{lem:ptf-alt-distr} (resp. Lemma \ref{app:lem:ptf-nul-distr}) analyzes 
	the properties of $(\bx',y')$ when the input LWE instance is from 
	the alternative hypothesis case (resp. null hypothesis case).
	Theorem \ref{thm:lwe-to-massart} analyzes the properties of $(V_d(\bx'),y')$ 
	for the input LWE instance from both cases.}  \label{fig:flow}
\end{figure}

\subsection{Illustration of Hard Instances} \label{app:figures}

For the sake of intuition, we additionally present the following figures 
to illustrate the ideas behind our construction. 
Figure \ref{fig:dk-construction} illustrates the original \cite{diakonikolas2021near} construction, 
as discussed in Section \ref{sec:main-reduction}.

\begin{figure}[H] 
	\includegraphics[width=14cm]{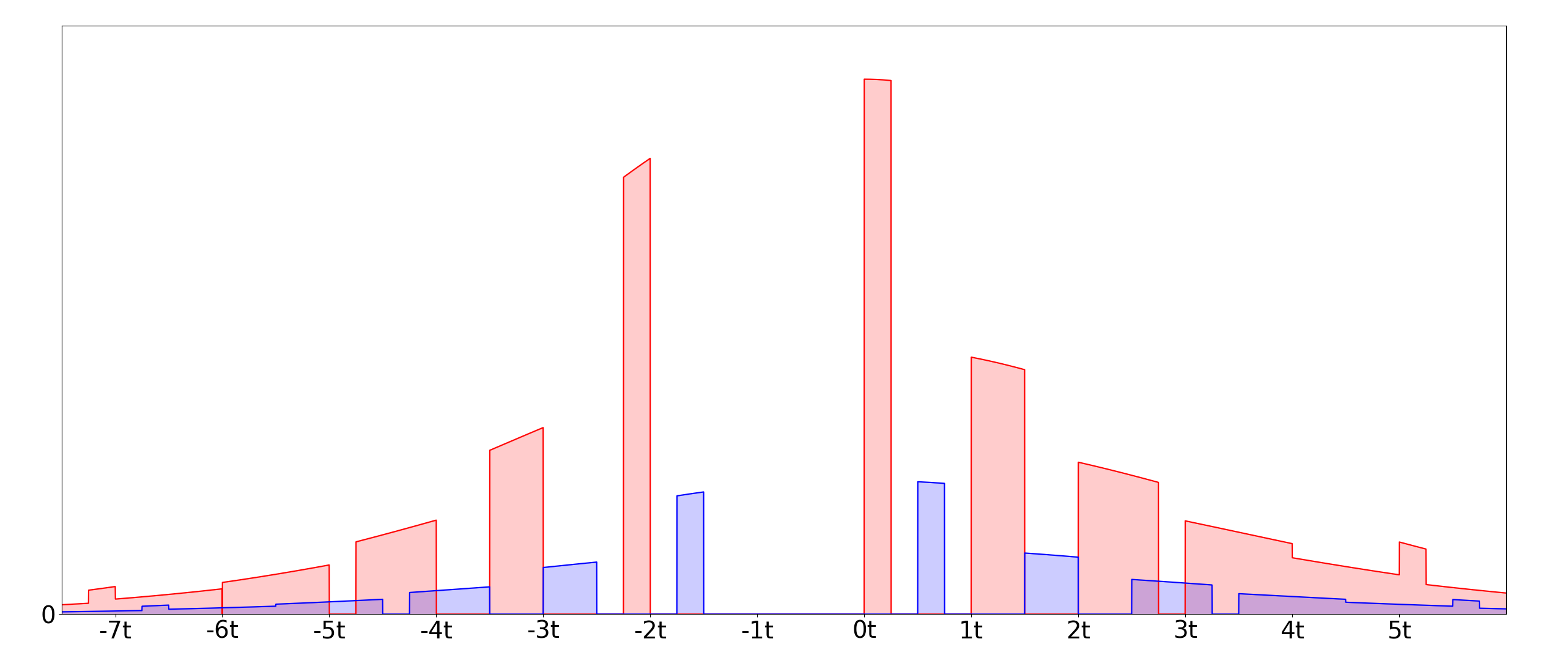}
	\caption{The original SQ-hard construction in \cite{diakonikolas2021near} 
	\newblue{(specifically, the univariate distribution in the hidden direction)}. 
    The red part corresponds to $\pr[y=+1]P_{\left [\bx^s\mid (y=+1)\right ]}(\cdot)$; 
    the blue part corresponds to $\pr[y=-1]P_{\left [\bx^s\mid (y=-1)\right ]}(\cdot)$.} 
    \label{fig:dk-construction}
\end{figure}
As discussed in Section \ref{subsec:carving-explaination}, 
if we try to replace the ``hidden direction discrete Gaussian'' with its noisy variant, 
we will get a construction as the one 
illustrated in Figure \ref{fig:dk-construction-w-noise}.
\begin{figure}[H] 
	\includegraphics[width=14cm]{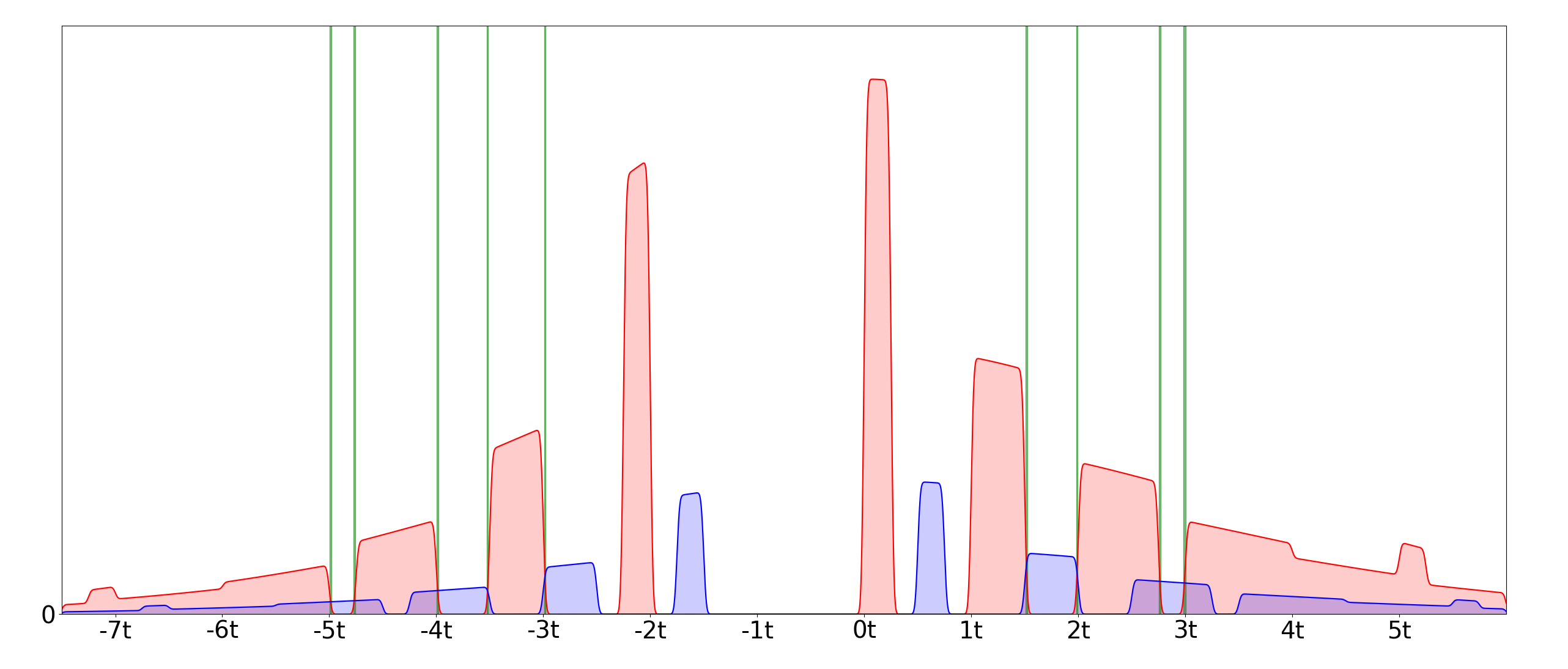}
	\caption{The \cite{diakonikolas2021near} construction with a noisy ``hidden direction discrete Gaussian''. 
	The red part corresponds to $\pr[y=+1]P_{\left [\bx^s\mid (y=+1)\right ]}(\cdot)$, 
	and the blue part corresponds to \mbox{$\pr[y=-1]P_{\left [\bx^s\mid (y=-1)\right ]}(\cdot)$}. 
	Notice the thin green intervals where 
	$\frac{\pr[y=+1]P_{\left [\bx^s\mid (y=+1)\right ]}(\cdot)}
	        {\pr[y=-1]P_{\left [\bx^s\mid (y=-1)\right ]}(\cdot)}$ is close to 1; 
	the labels $+1$ and $-1$ are close to equally likely in these regions, 
	and thus violate the Massart noise condition.
	(\newblue{Notice that, as we explained in Section \ref{subsec:carving-explaination}, 
	there are other places where $\frac{\pr[y=+1]P_{\left [\bx^s\mid (y=+1)\right ]}(\cdot)}
	{\pr[y=-1]P_{\left [\bx^s\mid (y=-1)\right ]}(\cdot)}$ is close to $1$; 
	however, the density of these regions is negligible.})}  
	\label{fig:dk-construction-w-noise}
\end{figure}

Our idea is to modify the \newblue{above} construction by carving empty 
slots on the support of $\bx^s\mid (y=-1)$, so the ``clean'' version 
of the construction (using the ``hidden direction discrete Gaussian'' without noise) 
is as presented in Figure \ref{fig:construction-wo-noise}. 
\begin{figure}[H] 
	\includegraphics[width=14cm]{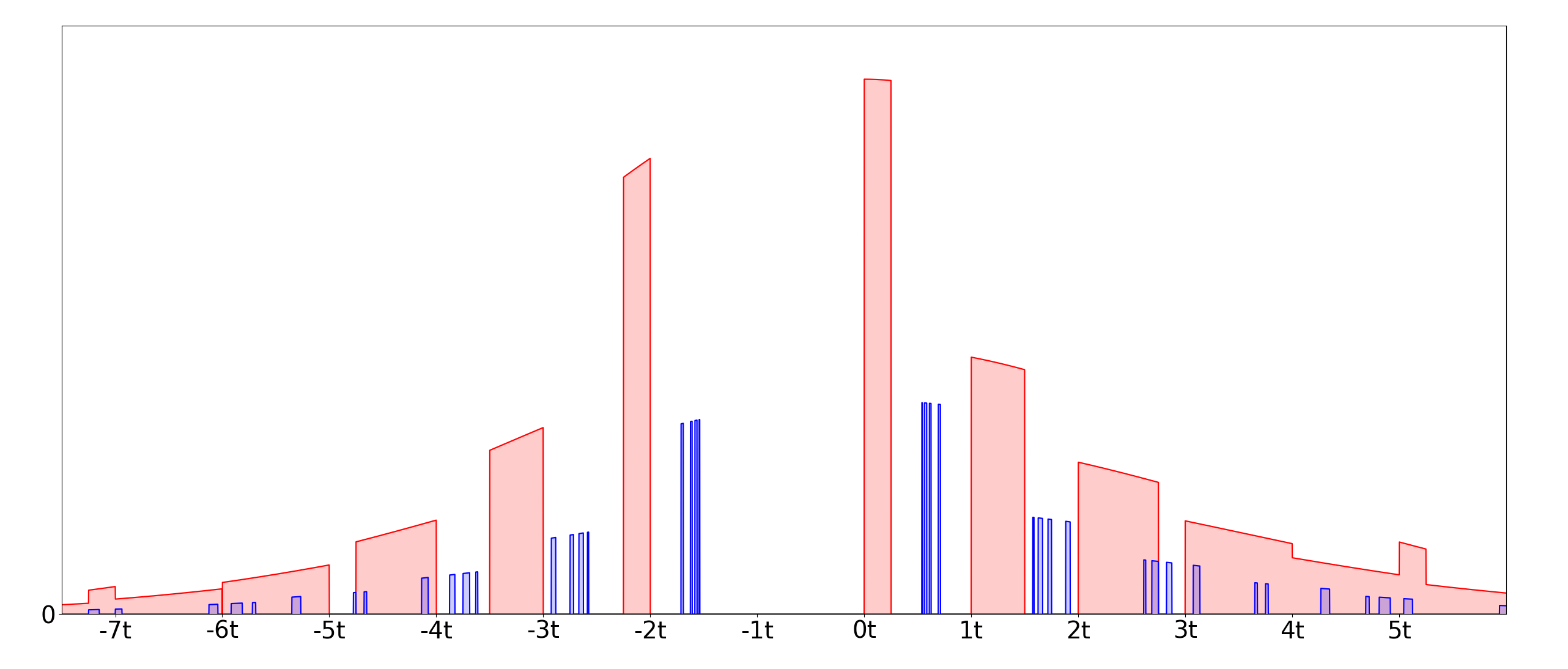}
	\caption{The modified hardness construction with the clean ``hidden direction discrete Gaussian''. 
	The red part corresponds to $\pr[y=+1]P_{\left [\bx^s\mid (y=+1)\right ]}(\cdot)$ and 
	the blue part corresponds to $\pr[y=-1]P_{\left [\bx^s\mid (y=-1)\right ]}(\cdot)$. 
	Compared with the original construction, 
	we carve empty slots on $\bx^s\mid (y=-1)$ for the problematic part to fit in. 
	Since the total mass we carve out is at most a constant fraction, 
	this causes the density to increase by at most a constant multiplicative factor.}  \label{fig:construction-wo-noise}
\end{figure}
With the noisy ``hidden direction discrete Gaussian'', the modified construction is as presented in 
Figure \ref{fig:construction-w-noise}.
\begin{figure}[H] 
	\includegraphics[width=14cm]{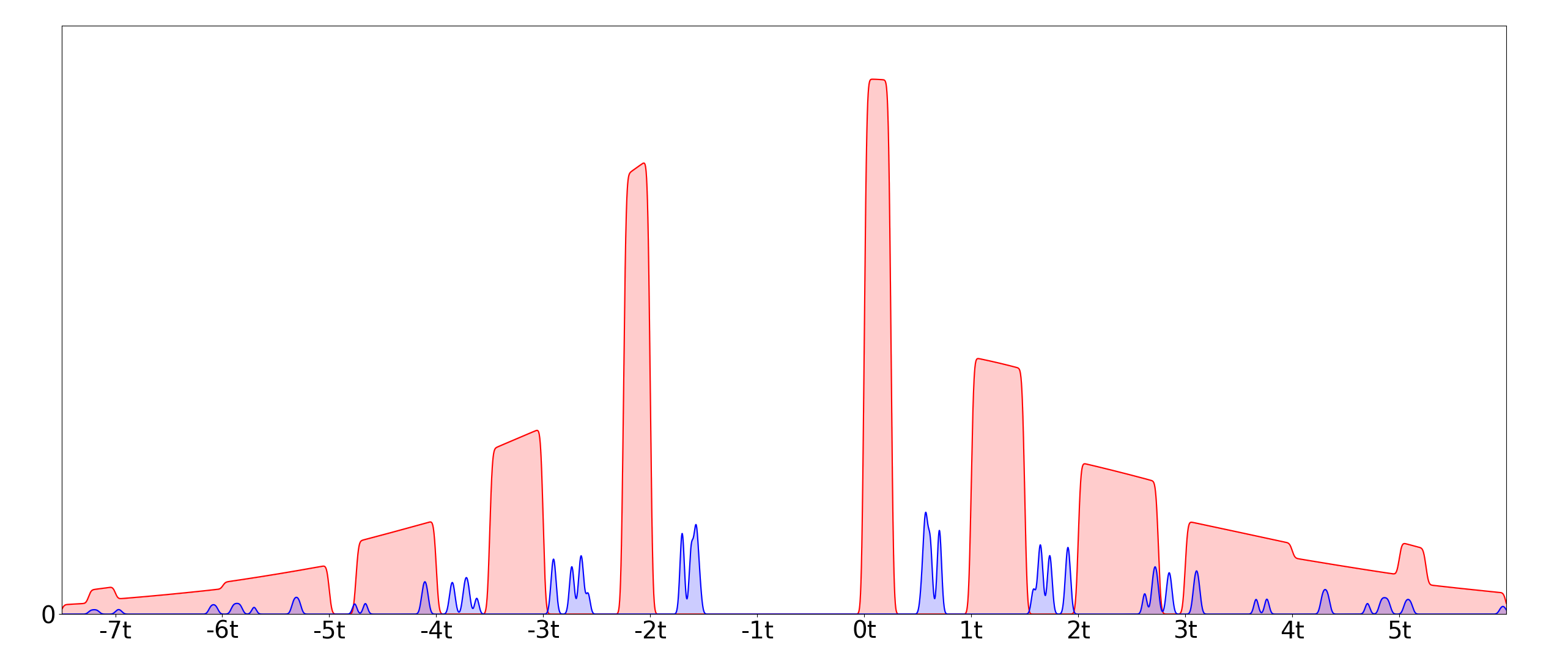}
	\caption{The modified construction with the noisy ``hidden direction discrete Gaussian''. 
	The red part corresponds to $\pr[y=+1]P_{\left [\bx^s\mid (y=+1)\right ]}(\cdot)$ 
	and the blue part corresponds to $\pr[y=-1]P_{\left [\bx^s\mid (y=-1)\right ]}(\cdot)$.}  \label{fig:construction-w-noise}
\end{figure}
\label{app:subsec:basic-rejection}
It is worth noting that even this modified construction does not perfectly satisfy the Massart condition. 
However, the part that violates the Massart condition is negligibly small, 
thus it is very close to a distribution that perfectly satisfies the Massart condition.

\subsection{Omitted Proofs from Section \ref{subsec:basic-rejection}} 
\newblue{
Here we give the proof for Lemma \ref{lem:basic-rs}.
Lemma \ref{lem:basic-rs} is stated only for the sake of intuition, 
and is not needed for the proof of our main theorem.
} 

\begin{proof} [Proof of Lemma \ref{lem:basic-rs}]
Notice $y=\mod_1(\langle \bs,\bx\rangle+z)$.
We rename the variables as $\bx_\nw\eqdef \bx$, $\sigma'_\scl\eqdef\sigma_\scl$ and $z_\nw=z$ 
then we have
$$\bx'\sim (1/\sigma'_\scl)\circ\Dgaus_{\bx_\nw+\Z^n,\sigma'_\scl}|(y=y')$$ where
$\bx_\nw\sim U(\R_1^n)$, $z_\nw\sim \Dgaus_\sigma$
are independent
and $y=\mod_1(\langle\bs,\bx_\nw\rangle+z_\nw)$.
This is the exact same distribution we calculated in 
the proof of Lemma \ref{lem:rejection-alternative-distribution} 
(after we rewrote the distribution in Lemma \ref{lem:rejection-alternative-distribution}). 
The same calculation gives the lemma statement here.
\end{proof}

\subsection{Omitted Proofs from Section \ref{subsec:complete-construction}} \label{app:subsec:complete-construction}

\subsubsection{Analysis of Algorithm \ref{algo:sampling}} 
\label{subsubsec:parameter-validity}

We first prove that the parameters defined in Algorithm \ref{algo:sampling} are well posed. 
In particular, we need to show that the expression inside the square root in Step~\ref{lne:output-sampling} of 
Algorithm \ref{algo:sampling} is positive.

\begin{observation} \label{app:obv:parameter-validity}
Let $\sr, \sigma_\scl$ be as in Algorithm~\ref{algo:sampling}. 
Then $\sr \geq 1/2$ and $(1 - \sr)\sigma_\scl^2 \geq \sr(\sigma/\sqrt{n})^2$.
\end{observation}

\begin{proof}
The first three conditions of the parameters in Condition~\ref{cond:params} 
yield $(t + \eps)\sigma \leq 2t\sigma \leq 2t\sqrt{n} \leq 2/\sqrt{c\log(n/\delta)}$, 
which is at most $1/8$ for sufficiently large $c$. This implies that $\sr \geq 1/2$.

Moreover, from our choice of $\sigma_\scl$, we have
\begin{align*}
\frac{(1 - \sr)\sigma_\scl^2}{\sr(\sigma/\sqrt{n})^2}
= \frac{\sr(1 - \sr)}{((t + k - \psi)\sigma)^2}
\geq \frac{(1/2)(4(t + \eps)^2\sigma^2)}{((t + \eps)\sigma)^2} 
= 2. & 
\end{align*}
\end{proof}
\newblue{
We note that Observation \ref{app:obv:parameter-validity} 
implies that the expression inside the square root in Step~\ref{lne:output-sampling} of 
Algorithm \ref{algo:sampling} is positive; 
therefore, $\sigma_\ad$ in Step~\ref{lne:output-sampling} is real.
}
It is also useful to observe some other properties 
of the parameters that will be useful later 
in the proof of Lemma \ref{lem:rejection-alternative-distribution}.

\begin{observation} \label{app:obv:parameter-property}
Let $\sr, \sigma_\scl$ be as in Algorithm~\ref{algo:sampling}. 
Then $\sigma_\scl\geq \frac{1}{2(t+k-\psi)\sqrt{n}}$ 
and $\sr=\frac{\sigma_\scl^2}{\sigma_\scl^2+\sigma_\ad^2+\sigma^2/n}$.
\end{observation}

\begin{proof}
The observation $\sigma_\scl\geq \frac{1}{2(t+k-\psi)\sqrt{n}}$ follows from 
$\sr\geq 1/2$ (Observation \ref{app:obv:parameter-validity}) and $\sigma_\scl= \frac{\sr}{(t+k-\psi)\sqrt{n}}$.
Then $\sr=\frac{\sigma_\scl^2}{\sigma_\scl^2+\sigma_\ad^2+\sigma^2/n}$ follows from the definition of 
$\sigma_\ad$ in Algorithm \ref{algo:sampling}.
\end{proof}

We next prove a lower bound on the acceptance probability of Algorithm \ref{algo:sampling}.

\begin{lemma} \label{app:lem:sampling-efficiency}
If $y \sim U(\R_1)$, then Algorithm \ref{algo:sampling} accepts with probability at least
$\Omega(\frac{\lambda(B)(t-\psi)}{t^2})$.
\end{lemma}

\begin{proof} 
The probability of a sample not being rejected by the first rejection in Step \ref{lne:first-rejection} of Algorithm \ref{algo:sampling} is 
$$\lambda \left (\left \{\frac{k}{t+(k-\psi)},k\in B\right \}\right ) 
\geq  \min_{k\in B}\left [\frac{d\left (\frac{k}{t+(k-\psi)}\right )}{dk}\right ]\cdot\lambda(B)
\geq \frac{\lambda(B)(t-\psi)}{(t+\eps)^2}=\Omega\left (\frac{\lambda(B)(t-\psi)}{t^2}\right ) \;,$$ 
where the inequality holds because $\frac{k}{t+(k-\psi)}$ is monotone increasing for $k\in B$. 
Conditioned on passing the first rejection, the second round rejects with probability $1-\frac{t^2}{(t+k-\psi)^2},$ 
which is at most $3/4$ since $t+k-\psi\leq t+\eps\leq 2t$.
Therefore, a sample is accepted with probability at least $\Omega\left (\frac{\lambda(B)(t-\psi)}{t^2}\right)$.
\end{proof}

\subsubsection{Proof of Lemma \ref{lem:rejection-alternative-distribution}}

\begin{proof} [Proof of Lemma \ref{lem:rejection-alternative-distribution}] 
{\bf We first prove property (i).}
We first review the definition of $D_{t,\epsilon,\psi,B,\delta}^\alt$. 
We have $(\bx,y)$ from the alternative hypothesis case of the LWE problem. 
Namely, this means we have an unknown $\mathbf{s}\in \{\pm 1\}^n$ 
and independently sampled $\bx\sim U(\R_1^n)$ and $z\sim \Dgaus_{\sigma}$. 
Then $y=\mod_1(\langle \bx,\mathbf{s}\rangle +z )$. 
We reject unless $y=\frac{k}{t+(k-\psi)}$ for some $k\in B$. 
Then the secondary rejection step rejects with probability $1-\frac{t^2}{(t+k-\psi)^2}$.
We calculate $\sigma_\ad$ and $\sigma_\scl$ which only depend on $y$ 
and output a sample
$$\bx'\sim (1/\sigma_\scl)\circ \Dgaus_{\bx+\bx_\ad+\Z^n,\sigma_\scl}\; ,$$
where $\bx_\ad\sim \Dgaus_{\R^n,\sigma_\ad}$.
Let $k$ be the value such that $y'=\frac{k}{t+(k-\psi)}$. Then since $y\sim U(\R_1)$
\begin{align*}
D_{t,\epsilon,\psi,B,\delta}^\alt
&\propto \int_0^1 \idt(k\in B)
\left (\frac{t^2}{(t+k-\psi)^2} \right )
\left [(1/\sigma_\scl)\circ \Dgaus_{\bx+\bx_\ad+\Z^n,\sigma_\scl}\; 
\bigg | \left (y=y'\right )\right ] dy'\\
&\propto \int_B 
\left [(1/\sigma_\scl)\circ \Dgaus_{\bx+\bx_\ad+\Z^n,\sigma_\scl}\; 
\bigg | \left (y=\frac{k}{t+(k-\psi)}\right )\right ] dk\;.
\end{align*}
Taking the proper scaling factor gives 
\begin{align*}
D_{t,\epsilon,\psi,B,\delta}^\alt
&= \frac{1}{\lambda(B)}\int_B 
\left [(1/\sigma_\scl)\circ \Dgaus_{\bx+\bx_\ad+\Z^n,\sigma_\scl}\; 
\bigg | \left (y=\frac{k}{t+(k-\psi)}\right )\right ] dk \;.
\end{align*}
We will prove that for any fixed \newblue{$k$}, 
\newblue{
letting 
$\bx_k'\sim (1/\sigma_\scl)\circ \Dgaus_{\bx+\bx_\ad+\Z^n,\sigma_\scl}\; 
\bigg | \left (y=\frac{k}{t+(k-\psi)}\right )$,
that
$$P_{\bx_k'^\bs}(u)=(1\pm O(\delta)) P_{\Dgaus_{k+(t+k-\psi)\Z,\sigma_\sn}\star \Dgaus_{\sigma_\ns}}(u)\; .$$
}

We use $\sigma_\ad'$ and $\sigma_\scl'$ to denote the specific values of $\sigma_\ad$ and $\sigma_\scl$
for $y=\frac{k}{t+(k-\psi)}$.
Then 
\begin{align*}
\newblue{D_{\bx_k'}}
&=(1/\sigma_\scl)\circ \Dgaus_{\bx+\bx_\ad+\Z^n,\sigma_\scl}\; \mid
\left (y=\frac{k}{t+(k-\psi)} \right )\\
&=(1/\sigma_\scl)\circ \Dgaus_{\mod_1(\bx+\bx_\ad)+\Z^n,\sigma_\scl}\; \mid
\left (y=\frac{k}{t+(k-\psi)} \right )\\
&=(1/\newblue{\sigma_\scl'})\circ \Dgaus_{\mod_1(\bx+\bx'_\ad)+\Z^n,\sigma'_\scl}\; \mid\left (y=\frac{k}{t+(k-\psi)} \right ) \;,
\end{align*}
where $\bx_\ad'\sim \Dgaus_{\R^n,\sigma_\ad'}$.
Now we will attempt to reason about this random process and replace it with something equivalent.
Letting $\bx_\nw=\mod_1(\bx+\bx'_\ad)$ and $z_\nw=-\langle \bx'_\ad,\mathbf{s}\rangle+z$, we notice that 
$$y=\mod_1(\langle \bx_\nw,\mathbf{s}\rangle +z_\nw)\;.$$ 
We show that $\bx_\nw$ and $z_\nw$ are independent before conditioning on $y$,
therefore we can instead consider the random process 
as sampling independent
$\bx_\nw$ and $z_\nw$ and then conditioning on $y$.
Note that $\bx_\nw=\mod_1(\bx+\bx'_\ad)\sim U(\R_1^n)$. 
If we condition on any fixed $\bx'_\ad$, the conditional distribution of $\bx_\nw$ is always $U(\R_1^n)$, 
therefore $\bx'_\ad$ and $\bx_\nw$ are independent. 
Since both $\bx'_\ad$ and $z$ are independent of $\bx_\nw$, we have that
$z_\nw=-\langle \bx'_\ad,\mathbf{s}\rangle+z$ is also independent of $\bx_\nw$. 
Therefore,
$$D_{\newblue{\bx_k'}}
=(1/\newblue{\sigma_\scl'})\circ \Dgaus_{\bx_\nw+\Z^n,\sigma_\scl'}\; \mid
\left (\mod_1(\langle \bx_\nw,\mathbf{s}\rangle +z_\nw)=y'\right )\; ,$$
where we can interpret $\bx_\nw$ and $z_\nw$ as independent samples
$\bx_\nw\sim U(\R_1^n)$ 
and $z_\nw\sim \Dgaus_{\sqrt{\sigma^2+\|\mathbf{s}\|_2^2\newblue{\sigma_{\ad}'^2}}}$.

Now let $\bw \sim (1/\newblue{\sigma_\scl'})\circ\Dgaus_{\bx_\nw+\Z^n,\sigma'_\scl}$. Since $\bx_\nw \sim U(\R_1^n)$, we have  $\Dgaus_{\bx_\nw+\Z^n,
\sigma_\scl'}
=D^{\mathrm{expand}}_{\R_1^n,\sigma_\scl'}\;$.
\newblue{
From the lower bound of $\sigma_\scl$ in Observation \ref{app:obv:parameter-property} and Condition (iii) in Condition \ref{cond:params}, we have 
$$\sigma_\scl'\geq \frac{1}{2(t+k-\psi)\sqrt{n}}\geq \frac{1}{4t\sqrt{n}}\geq \sqrt{c\log (n/\delta)} \;.$$}
Therefore, applying Fact \ref{app:fct:bound-on-collaped-distribution} yields $$P_{\bw}(\bu)=(1\pm O(\delta)) P_{\Dgaus_{\R^n,1}}(\bu)\; .$$
Since $\mathbf{s}\in \{\pm 1\}^n$ and $\mod_1(\bx_\nw)=\mod_1(\newblue{\sigma_\scl'}\mathbf{w})$, we have that
$$\mod_1(\langle \bx_\nw,\bs\rangle )=\mod_1(\langle \newblue{\sigma_\scl'}\mathbf{w},\bs\rangle )\; .$$ 
We can now rewrite the distribution as
\begin{align*}
    D_{\newblue{\bx_k'}}
    &=(1/\newblue{\sigma_\scl'})\circ\Dgaus_{\mod_1(\bx+\bx'_\ad)+\Z^n,\sigma'_\scl}\;  \mid\left (y=\frac{k}{t+(k-\psi)} \right )\\
	&=(1/\newblue{\sigma_\scl'})\circ\Dgaus_{\bx_\nw+\Z^n,\sigma'_\scl}\;  \mid
    \left (\mod_1(\langle \bx_\nw,\mathbf{s}\rangle +z_\nw)=\frac{k}{t+(k-\psi)}\right )\\
	&=\mathbf{w}\; \big |\left (\mod_1(\langle \newblue{\sigma_\scl'}\bw,\mathbf{s}\rangle +z_\nw)
	=\frac{k}{t+(k-\psi)}\right ) \\
	&=\mathbf{w}\; \big |\left (\mod_1(\newblue{\sigma_\scl'}\|\bs\|_2\bw^\bs +z_\nw)
	=\frac{k}{t+(k-\psi)}\right )\; ,
\end{align*}
where 
$$P_{\mathbf{w}}(\mathbf{u})=(1\pm O(\delta))P_{\Dgaus_{\R^n,1}}(\mathbf{u})\; ,$$
and
$$z_\nw\sim \Dgaus_{\sqrt{\sigma^2+\|\mathbf{s}\|_2^2\newblue{\sigma_\ad'^2}}}$$ 
are independently sampled (this follows from $\bx_\nw$ and $z_\nw$ being independent, since $\bw$ solely
depends on $\bx_\nw$).
The above form allows us the explicitly 
express the PDF function as follows:
\begin{align*}
P_{\newblue{\bx_k'^\bs}}(u)
\propto & P_{\bw^\bs}(u)\sum_{i\in \Z+\frac{k}{t+(k-\psi)}}P_{z_\nw}(i-\newblue{\sigma_\scl'}\|\bs\|_2 u)\\
\propto & (1\pm O(\delta))\sum_{i\in \Z+\frac{k}{t+(k-\psi)}}P_{\Dgaus_1}(u)
	P_{\Dgaus_{\sqrt{\sigma^2+\|\bs\|_2^2\newblue{\sigma_{\ad}'^2}}}}(i-\newblue{\sigma_\scl'}\|\bs\|_2 u)\\
\propto & (1\pm O(\delta))\sum_{i\in \newblue{\sigma_\scl'^{-1}}n^{-1/2}\left (\Z+\frac{k}{t+(k-\psi)}\right )}P_{\Dgaus_1}(u)
	P_{\Dgaus_{\sqrt{\frac{\sigma^2/n+\newblue{\sigma_\ad'^2}}{\newblue{\sigma_\scl'^2}}}}}(i-u) \;.\\
\end{align*}
Letting $\alpha\eqdef \newblue{\sqrt{\frac{\sigma^2/n+\sigma_\ad'^2}{\sigma_\scl'^2}}}$, we notice that
$\alpha=\sqrt{1/\sr-1}$ \newblue{from Observation \ref{app:obv:parameter-property}}. Thus, we can write
\newblue{
\begin{align*}
P_{\newblue{\bx_k'^\bs}}(u)
\propto & (1\pm O(\delta))\sum_{i\in \newblue{\sigma_\scl'^{-1}}n^{-1/2}\left (\Z+\frac{k}{t+(k-\psi)}\right )}P_{\Dgaus_1}(u)
	P_{\Dgaus_{\alpha}}(i-u)\\
\propto & (1\pm O(\delta))\sum_{i\in \newblue{\sigma_\scl'^{-1}}n^{-1/2}\left (\Z+\frac{k}{t+(k-\psi)}\right )}
	\exp\left (-\pi\left (u^2+\frac{(i-u)^2}{\alpha^2}\right )\right )\\
\propto & (1\pm O(\delta))\sum_{i\in \newblue{\sigma_\scl'^{-1}}n^{-1/2}\left (\Z+\frac{k}{t+(k-\psi)}\right )}
	\exp\left (-\pi\left (\frac{(\alpha^2+1)u^2}{\alpha^2}+\frac{i^2}{\alpha^2}-\frac{2iu}{\alpha^2}\right )\right )\\
\propto & (1\pm O(\delta))\sum_{i\in \newblue{\sigma_\scl'^{-1}}n^{-1/2}\left (\Z+\frac{k}{t+(k-\psi)}\right )}
	\exp\left (-\pi\left (\frac{i^2}{\alpha^2+1}+\frac{\left ((\alpha^2+1)^{1/2}u-(\alpha^2+1)^{-1/2}i\right )^2}{\alpha^2}\right )\right )\\
\propto & (1\pm O(\delta))\sum_{i\in \sigma_\scl'^{-1}n^{-1/2}\left (\Z+\frac{k}{t+(k-\psi)}\right )}
	\exp\left (-\pi\left (\frac{\left ((\alpha^2+1)^{-1}i\right )^2}{(\alpha^2+1)^{-1}}+\frac{\left (u-(\alpha^2+1)^{-1}i\right )^2}{\alpha^2(\alpha^2+1)^{-1}}\right )\right )\\
\propto & (1\pm O(\delta))\sum_{i\in \sigma_\scl'^{-1}n^{-1/2}\left (\Z+\frac{k}{t+(k-\psi)}\right )}
	P_{\Dgaus_{(\alpha^2+1)^{-1/2}}}((\alpha^2+1)^{-1}i) P_{\Dgaus_{\alpha(\alpha^2+1)^{-1/2}}}(u-(\alpha^2+1)^{-1}i)   \; .
\end{align*} }
Notice that $(\alpha^2+1)^{-1}=\sr$ from Observation \ref{app:obv:parameter-property}. Thus, we get
\begin{align*}
P_{\newblue{\bx_k'^\bs}}(u)
    \propto & (1\pm O(\delta))\sum_{i\in \sigma_\scl^{-1}n^{-1/2}
    \left (\Z+\frac{k}{t+(k-\psi)}\right )}
    P_{\Dgaus_{\sqrt{\sr}}}(\sr i)
	P_{\Dgaus_{\sqrt{1-\sr}}}(\sr i-u)\\    
    \propto & (1\pm O(\delta))\sum_{i\in \sr\sigma_\scl^{-1}n^{-1/2}
    \left (\Z+\frac{k}{t+(k-\psi)}\right )}
    P_{\Dgaus_{\sqrt{\sr}}}(i)
	P_{\Dgaus_{\sqrt{1-\sr}}}(i-u)\; .
\end{align*}
Since $\sigma_\scl = \frac{\sr}{(t + k - \psi)\sqrt{n}}$ from Step~\ref{lne:output-sampling} of Algorithm \ref{algo:sampling},
\begin{align*}
P_{\newblue{\bx_k'^\bs}}(u)
    \propto & (1\pm O(\delta))\sum_{i\in k+(t+k-\psi)\Z}
    P_{\Dgaus_{\sqrt{\sr}}}(i)
	P_{\Dgaus_{\sqrt{1-\sr}}}(i-u)\; .
\end{align*}
\newblue{
Notice the expression here is propotional to the convolution of $\Dgaus_{k+(t+k-\psi)\Z,\sqrt{\sr}}$ and $\Dgaus_{\sqrt{1-\sr}}$. Therefore, we have 
\begin{align*}
	P_{\bx_k'^\bs}(u)&=(1\pm O(\delta)) P_{\Dgaus_{k+(t+k-\psi)\Z,\sqrt{\sr}}\star \Dgaus_{\sqrt{1-\sr}}}(u)\\
	&=(1\pm O(\delta)) P_{\Dgaus_{k+(t+k-\psi)\Z,\sigma_\sn}\star \Dgaus_{\sigma_\ns}}(u)\; .
\end{align*}

Now we can plug this back and calculate 
$\left [D_{t,\epsilon,\psi,B,\delta}^\alt \right ]^\bs$,
as follows:
\begin{align*}
P_{\left [D_{t,\epsilon,\psi,B,\delta}^\alt \right ]^\bs}(u)
&=
\frac{1}{\lambda(B)}\int_B P_{\bx_k'^\bs}(u) dk\\
&=
\frac{1}{\lambda(B)}(1\pm O(\delta)) \int_B P_{\Dgaus_{k+(t+k-\psi)\Z,\sigma_\sn}\star \Dgaus_{\sigma_\ns}}(u) dk\; .
\end{align*}
Note that inside the integration is the convolution of two distributions, $\Dgaus_{k+(t+k-\psi)\Z,\sigma_\sn}$ and $\Dgaus_{\sigma_\ns}$.
Since $\Dgaus_{\sigma_\ns}$ is independent of $k$, we can interpret it as added after the integration.
Therefore, 
\begin{align*}
P_{\left [D_{t,\epsilon,\psi,B,\delta}^\alt \right ]^\bs}(u)
&=(1\pm O(\delta)) P_{\left [\int_B \frac{1}{\lambda(B)}\Dgaus_{k+(t+k-\psi)\Z,\sigma_\sn} dk\star \Dgaus_{\sigma_\ns}\right ]} (u)\\
&=(1\pm O(\delta)) P_{D'\star \Dgaus_{\sigma_\ns}} (u)\; ,
\end{align*}
where	
$$\sigma_\sn=\sqrt{\sr}\; ,$$
and
$$\sigma_\ns=\sqrt{1-\sr} = 2(t+\eps)\sigma\; .$$
This proves property (i).
}

{\bf Now we prove Property (ii).} 
We prove that for any fixed $y=\frac{k}{t+(k-\psi)}$, 
letting $\bx_k'\sim D_{t,\epsilon,\psi,B,\delta}^\alt\mid \left (y=\frac{k}{t+(k-\psi)}\right )$, 
then $\bx_k'^{\perp\bs}$ is nearly 
independent of $\bx_k'^\bs$ 
in the sense that 
for any $l\in \R$ and $\bu\in \R^{n-1}$, we have that
$$P_{\bx_k'^{\perp \bs}\mid\bx_k'^\bs=l}(\bu)
=(1\pm O(\delta)) P_{\Dgaus_{\R^{n-1},1}}(\bu)\;.$$
Following the same notation we used for proving Property (i), 
we can rewrite the random process as
\begin{align*}
    (1/\sigma_\scl)\circ\Dgaus_{\bx+\bx_\ad+\Z^n,\sigma_\scl}\; \mid
    \left (y=\frac{k}{t+(k-\psi)} \right )
	=\mathbf{w}\; \big |\left (\mod_1(\newblue{\sigma_\scl'}\|\bs\|_2\bw^\bs +z_\nw)
	=\frac{k}{t+(k-\psi)}\right )\; ,
\end{align*}
where 
$$P_{\mathbf{w}}(\mathbf{u})=(1\pm O(\delta))P_{\Dgaus_{\R^n,1}}(\mathbf{u})\; ,$$
and
$$z_{\nw}\sim \Dgaus_{\sqrt{\sigma^2+\|\mathbf{s}\|_2^2\newblue{\sigma_{\ad}'^2}}}$$ 
are independently sampled (before conditioning).
Therefore, we have that
\begin{align*}
P_{\bx_k'^{\perp \bs}\mid\bx_k'^\bs=l}(\bu)
&=P_{\left [\bw^{\perp \bs}\mid \left (\mod_1(\newblue{\sigma_\scl'}\|\bs\|_2\bw^\bs +z_\nw)
=\frac{k}{t+(k-\psi)}\right )\land (\bw^\bs=l)\right ]}(\bu)\\
&=P_{\left [\bw^{\perp \bs}\mid \left (\mod_1(\newblue{\sigma_\scl'}\|\bs\|_2l +z_\nw)
=\frac{k}{t+(k-\psi)}\right )\land (\bw^\bs=l)\right ]}(\bu)\\
&=P_{\left [\bw^{\perp \bs}\mid \bw^\bs=l\right ]}(\bu)\; .
\end{align*}
Recall that the PDF of $\bw$ is pointwise 
close to a Gaussian, as shown earlier in proof for Property (i),
therefore, we obtain
$$P_{\bx_k'^{\perp \bs}\mid\bx_k'^\bs=l}(\bu)
=P_{\bw^{\perp \bs}\mid\bw^\bs=l}(\bu)
=(1\pm O(\delta)) P_{\Dgaus_{\R^{n-1},1}}(\bu) \;.$$
Considering that $\bx'\sim D_{t,\epsilon,\psi,B,\delta}^\alt$ 
is a mixture of
$\bx_k'\sim D_{t,\epsilon,\psi,B,\delta}^\alt\mid \left (y=\frac{k}{t+(k-\psi)}\right )$
for different values of $k$, it follows that
$$P_{\bx'^{\perp \bs}\mid\bx'^\bs=l}(\bu)
=(1\pm O(\delta)) P_{\Dgaus_{\R^{n-1},1}}(\bu)\;.$$
\end{proof}

We also give the following claim obtained from Lemma \ref{lem:rejection-alternative-distribution}, 
which will be useful when we prove the 
Massart condition in Lemma \ref{lem:ptf-alt-distr}.

\begin{claim} \label{app:clm:ptf-alt-distr}
The exact PDF function of the distribution $D'$ in Lemma \ref{lem:rejection-alternative-distribution} is
\newblue{
$$P_{D'}(u)
=\frac{\Theta(t)}{\lambda(B)}
\sum_{i\in \Z}\frac{1}{|i+1|}\mathbf{1}(u\in it+\psi +(i+1)(B-\psi))\rho_{\sigma_\sn}(u)\; .$$
}
\end{claim}

\begin{proof}
This follows from expanding the expression of $D'$ in Lemma \ref{lem:rejection-alternative-distribution}. We have
\begin{align*}
    P_{D'}(u)
    &=P_{\left [\frac{1}{\lambda(B)}\int_B \Dgaus_{k+(t+k-\psi)\Z,\sigma_\sn} dk\right ]}(u)\\
    &=\frac{1}{\lambda(B)}\int_B P_{\left [\Dgaus_{k+(t+k-\psi)\Z,\sigma_\sn} \right ]}(u)dk\; .
\end{align*}    
Notice that 
\begin{align*}
P_{\left [\Dgaus_{k+(t+k-\psi)\Z,\sigma_\sn}\right ]}(u)
=
\idt\left (u\in k+(t+k-\psi)\Z\right )
\frac{\rho_{\sigma_\sn}(u)}
{\rho_{\sigma_\sn}(k+(t+k-\psi)\Z)}\;.
\end{align*}
Then using that
$\sigma_\sn=\sqrt{\sr}\geq 1/2$ (Observation \ref{app:obv:parameter-validity}) and 
$\frac{1}{t\sqrt{n}}\geq \sqrt{c\log(n/\delta)}$ for a sufficiently large $c$ (Condition \ref{cond:params}), 
we have that $\sigma_\sn/(t+k-\psi)\geq \sqrt{\tilde c\log(1/\delta)}$ for a sufficiently large $\tilde c$.
After that, an application of Lemma~\ref{app:lem:general-pointwise-bound} 
gives that
$$\rho_{\sigma_\sn}(k+(t+k-\psi)\Z)=(t+k-\psi)^{-1}\rho_{\sigma_\sn/(t+k-\psi)}(k/(t+k-\psi)+\Z)=(1\pm O(\delta))(t+k-\psi)^{-1}\; .$$
Plugging this back, we obtain
\begin{align*}
    P_{D'}(u)
    &=\frac{\newblue{1\pm O(\delta)}}{\lambda(B)}\int_B (t+k-\psi)
    \idt\left (u\in k+(t+k-\psi)\Z\right )\rho_{\sigma_\sn}(u)dk\\
    &=\frac{\newblue{1\pm O(\delta)}}{\lambda(B)}\int_B (t+k-\psi)\sum_{i\in \Z}\left [\idt (u=k+(t+k-\psi)i)\right ] 
    \rho_{\sigma_\sn}(u)dk\\
    &=\frac{\Theta(t)}{\lambda(B)} \sum_{i\in \Z}\left [\int_B\idt (u=k+(t+k-\psi)i)dk\right ] \rho_{\sigma_\sn}(u)\\
    &=\frac{\Theta(t)}{\lambda(B)}
    \sum_{i\in \Z}\frac{1}{|i+1|}\mathbf{1}(u\in it+\psi +(i+1)(B-\psi))\rho_{\sigma_\sn}(u)\;. \qedhere
\end{align*}
\end{proof}

\subsubsection{Proof of Lemma \ref{lem:rejection-null-distribution}}

\begin{proof} [Proof of Lemma \ref{lem:rejection-null-distribution}]
We first review the definition of $D_{t,\epsilon,\psi,B,\delta}^\nul$. 
Let $(\bx,y)$ be drawn from the null hypothesis 
case of the LWE problem. 
This means that we have $\bx\sim U(\R_1^n)$ and $y\sim U(\R_1)$ independently. 
In this case, we reject unless $y=\frac{k}{t+(k-\psi)}$ for some $k\in B$. 
Then the secondary rejection step rejects with probability $1-\frac{t^2}{(t+k-\psi)^2}$.
We calculate $\sigma_\ad$ and $\sigma_\scl$ which only depend on $y$, 
and output a sample
$$\bx'\sim (1/\sigma_\scl)\circ\Dgaus_{\bx+\bx_\ad+\Z^n,\sigma_\scl}\; ,$$
where $\bx_\ad\sim \Dgaus_{\R^n,\sigma_\ad}$.
Therefore, the distribution $D_{t,\epsilon,\psi,B,\delta}^\nul$ is a mixture of
$$(1/\sigma_\scl)\circ\Dgaus_{\bx+\bx_\ad+\Z^n,\sigma_\scl}\; \mid\left (y=y' \right )\; ,$$
for different values of $y'$.

We prove that for any fixed $y'$, the following holds:	
$$P_{\left [(1/\sigma_\scl)\circ\Dgaus_{\bx+\bx_\ad+\Z^n,\sigma_\scl}\mid (y=y')\right ]}(\bu)=(1\pm O(\delta))\cdot  P_{\Dgaus_{\R^n,1}}(\bu) \;.$$
We use $\sigma_\ad'$ and $\sigma_\scl'$ to denote the specific values of $\sigma_\ad$ and $\sigma_\scl$
for $y=y'$.
Then 
\begin{align*}
(1/\sigma_\scl)\circ\Dgaus_{\bx+\bx_\ad+\Z^n,\sigma_\scl}\; \mid
\left (y=y' \right )
=&(1/\sigma_\scl)\circ\Dgaus_{\mod_1(\bx+\bx_\ad)+\Z^n,\sigma_\scl}\; \mid
\left (y=y' \right )\\
=&(1/\newblue{\sigma_\scl'})\circ\Dgaus_{\mod_1(\bx+\bx'_\ad)+\Z^n,\sigma'_\scl}\; \mid\left (y=y' \right )\\
=&(1/\newblue{\sigma_\scl'})\circ\Dgaus_{\mod_1(\bx+\bx'_\ad)+\Z^n,\sigma'_\scl} \;,
\end{align*}
where $\bx_\ad'\sim \Dgaus_{\R^n,\sigma_\scl'}$.
Since $\bx$ is always drawn from $U(\R_1^n)$ and 
$\bx_\ad'$ is independent of $\bx$, then 
$\mod_1(\bx+\bx'_\ad)$ is also drawn from $U(\R_1^n)$. 
Thus, we can write 
$$\Dgaus_{\bx+\bx'_\ad+\Z^n,
\sigma_\scl'^2}
=D^{\mathrm{expand}}_{\R_1^n,\sigma_\scl'} \;.$$
\newblue{Applying Fact \ref{app:fct:bound-on-collaped-distribution}, the lower bound of $\sigma_\scl$ in Observation \ref{app:obv:parameter-property} and Condition (iii) in Condition \ref{cond:params}} 
yields
$$P_{D^{\mathrm{expand}}_{\R_1^n,\sigma_\scl'}}(\bu)
=(1\pm O(\delta))\cdot  P_{\Dgaus_{\R^n,\sigma_\scl'}}(\bu) \;.$$
Plugging this back, we obtain that
$$P_{\left [(1/\sigma_\scl)\circ\Dgaus_{\bx+\bx_\ad+\Z^n,\sigma_\scl}\mid (y=y')\right ]}(\bu)
=(1\pm O(\delta))\cdot  P_{\Dgaus_{\R^n,1}}(\bu) \;.$$
Then, since $D_{t,\epsilon,\psi,B,\delta}^\nul$ 
is a mixture of 
$(1/\sigma_\scl)\circ\Dgaus_{\bx+\bx_\ad+\Z^n,\sigma_\scl}\mid (y=y')$ 
for different $y'$, we get that
$$P_{D_{t,\epsilon,\psi,B,\delta}^\nul}(\bu)=(1\pm O(\delta))\cdot  P_{\Dgaus_{\R^n,1}}(\bu) \;.$$
\end{proof}

\subsubsection{Proof of Lemma \ref{lem:ptf-alt-distr}}

Before we prove Lemma \ref{lem:ptf-alt-distr}, we restate the definition for $B_-$ as a refresher.
$B_-$ is defined by carving out $O(t/\eps)$ many empty slots in $[t/2,t]$ in order to make the problematic part of $\bx^{\bs}|y=+1$ to fit in (see Figure \ref{fig:construction-wo-noise} and Figure \ref{fig:construction-w-noise}).
To define $B_-$, we need to first define a mapping $g$ that maps
a location to the corresponding place we need to carve out on $B_-$. The function $g:\R-[-1.5t,0.5t]\mapsto [0.5t, t]$ is defined as follows: for $i\in \Z$ and $b\in \R_t$, 
we have that
\begin{align*}
    g(it+t/2+b)\eqdef
    \begin{cases}
        \frac{b}{i+1}+t/2 &\text{if $i\geq 0$;}\\
        \frac{b-t}{i+2}+t/2 &\text{if $i<0$.}
    \end{cases}
\end{align*}
Then $B_-$ is defined as follows:
\begin{align*}
    B_-\eqdef &[t/2,t/2+\epsilon] \\
    &-\bigcup\limits_{i={\frac{t}{2\eps}-1}}^{\frac{t}{\eps}-1}g([it-2c'\epsilon,it])
	-\bigcup\limits_{i={\frac{t}{2\eps}-1}}^{\frac{t}{\eps}-1} 
	g([it+(i+1)\epsilon,it+(i+1)\epsilon+2c'\epsilon])\\
	&-\bigcup\limits_{i={-\frac{t}{\eps}}-1}^{-\frac{t}{2\eps}-1} 
	g([it+(i+1)\epsilon-2c'\epsilon,it+(i+1)\epsilon])
	-\bigcup\limits_{i={-\frac{t}{\eps}}-1}^{-\frac{t}{2\eps}-1} 
	g([it,it+2c'\epsilon])\; .
\end{align*}

\begin{proof} [Proof of Lemma \ref{lem:ptf-alt-distr}]
We will prove that there is a distribution $D^\tru$ such that \\$\dtv(D^\tru, D_\PTF^\alt)=O(\delta/m')$ and \newblue{there is a degree-$O(t/\eps)$}
PTF $\sgn(p(\cdot))$ such that
\begin{enumerate}
    \item $\pr_{(\bx,y)\sim D^\tru}[\sgn(p(\bx))\neq y]=\exp(-\Omega(t^4/\eps^2))$; and,
    \item $D^\tru$ satisfies the $O(\eta)$ Massart condition with respect to $\sgn(p(\bx))$.
\end{enumerate}

\newblue{We first give some high level intuition for $D^\tru$. }
First, we recall that $D_\PTF^\alt$ is the $1-\eta:\eta$ mixture of
$D_{t,\epsilon,\psi_+,B_+,\delta}^\alt$ and $D_{t,\epsilon,\psi_-,B_-,\delta}^\alt$ 
with $+1$ and $-1$ labels, respectively
\footnote{\newblue{
That is, $D_\PTF^\alt$ is the joint distribution of $(\bx,y)$ such that 
with probability $1-\eta$ we sample $\bx\sim D_{t,\epsilon,\psi_+,B_+,\delta}^\alt$ and $y=+1$; 
and with probability $\eta$ we sample $\bx\sim D_{t,\epsilon,\psi_-,B_-,\delta}^\alt$ and $y=-1$}}.
Also we recall that both $D_{t,\epsilon,\psi_+,B_+,\delta}^\alt$
and $D_{t,\epsilon,\psi_-,B_-,\delta}^\alt$ 
are noisy ``hidden direction discrete Gaussians'' which
are close to continuous Gaussian on all except the hidden direction $\bs$, and on the hidden direction, they are close
to a linear combination of discrete Gaussian plus an extra continuous Gaussian noise (as shown in Lemma \ref{lem:rejection-alternative-distribution}). 
The idea here is to truncate that extra continuous Gaussian noise on the hidden direction to obtain $D^\tru$.  

Now we formally define $D^\tru$ in the following way. 
We will first define distributions 
$D^\tru_+$ and $D^\tru_-$ such that
$$\dtv(D^\tru_+,D_{t,\epsilon,\psi_+,B_+,\delta}^\alt)=O(\delta/m')\;,$$
and
$$\dtv(D^\tru_-,D_{t,\epsilon,\psi_-,B_-,\delta}^\alt)=O(\delta/m')\; .$$
Then we take $D^\tru$ as the $1-\eta:\eta$ mixture of 
$D^\tru_+$ and $D^\tru_-$ with $+1$ and $-1$ labels respectively.

We define $D^\tru_+$ below; $D^\tru_-$ is defined analogously.
We specify $D^\tru_+$ to satisfy the following requirement:
Let $\bx'\sim D^\tru_+$ and $\bx\sim D_{t,\epsilon,\psi_+,B_+,\delta}^\alt$. 
For any $l\in \R$, we specify the conditional distribution of 
\newblue{$P_{\bx'^{\perp \bs}\mid \bx'^\bs=l}$} to be:
\begin{align*}
\newblue{P_{\bx'^{\perp \bs}\mid \bx'^\bs=l}=P_{\bx^{\perp \bs}\mid \bx^\bs=l}}\;.
\end{align*}
Since the conditional distributions on $\perp \bs$ are the same, to define $D^\tru_+$,
it remains to specify $[D^\tru_+]^\bs$.
According to Lemma \ref{lem:rejection-alternative-distribution} 
property (i),
$\left [ D_{t,\epsilon,\psi_+,B_+,\delta}^\alt\right ]^\bs$
is \newblue{pointwise close to} the convolutional sum of a distribution $D'$ 
and noise drawn from $\Dgaus_{\sigma_\ns}$.
\newblue{Here we use $D'_+$ (resp. $D'_-$) to denote the corresponding $D'$ for $\left [ D_{t,\epsilon,\psi_+,B_+,\delta}^\alt\right ]^\bs$ (resp. $\left [ D_{t,\epsilon,\psi_-,B_-,\delta}^\alt\right ]^\bs$). 
This is $$P_{\left [ D_{t,\epsilon,\psi_+,B_+,\delta}^\alt\right ]^\bs}(u)
=(1\pm O(\delta))P_{D'_+\star \Dgaus_{\sigma_\ns}}(u)\; .$$}
We define a truncated noise distribution $\gaus^\tru$ 
to replace $\Dgaus_{\sigma_\ns}$, as follows:
\begin{align*}
P_{\mathcal{N}^{\mathrm{truncate}}}(u)\propto
	\begin{cases}
		P_{\Dgaus_{\sigma_\ns}} (u) & \text{if $|u|\leq c'\epsilon$,}\\
		0 & \text{ otherwise} \;,
	\end{cases}
\end{align*}
where $c'$ is the parameter in \newblue{Condition \ref{cond:params}}.
We now define $[D^\tru_+]^\bs$ as \newblue{pointwise close to} the convolution sum of the distribution $D'_+$ 
and noise from $\gaus^\tru$, 
\newblue{
namely, $$P_{[D^\tru_+]^\bs}(u)\propto 
\frac{P_{\left [ D_{t,\epsilon,\psi_+,B_+,\delta}^\alt\right ]^\bs}(u)}{P_{D'_+\star \Dgaus_{\sigma_\ns}}(u)}P_{D'_+\star \gaus^\tru}(u)\; ,$$
where
$\frac{P_{\left [ D_{t,\epsilon,\psi_+,B_+,\delta}^\alt\right ]^\bs}(u)}{P_{D'_+\star \Dgaus_{\sigma_\ns}}(u)}=(1\pm O(\delta))$.}
$D^\tru_-$ is defined in the same manner.  

\newblue{
Now we bound $\dtv(D^\tru_+,D_{t,\epsilon,\psi_+,B_+,\delta}^\alt)$.
We first define
$$f(u)=\frac{P_{\left [ D_{t,\epsilon,\psi_+,B_+,\delta}^\alt\right ]^\bs}(u)}{P_{D'_+\star \Dgaus_{\sigma_\ns}}(u)}P_{D'_+\star \gaus^\tru}(u)\; ,$$
and notice that 
$$P_{D^\tru_+}(u)=\left (\int_{\R} f(u) du\right )^{-1} f(u)\; .$$
We then bound $\int_{\R} f(u) du$. Note that 
\begin{align*}
	\int_{\R} \left |f(u)-P_{\left [ D_{t,\epsilon,\psi_+,B_+,\delta}^\alt\right ]^\bs}(u)\right |du
	&= (1\pm O(\delta)) \int_{\R}\left |\int_{\R}P_{D_+}(t)P_{\gaus^\tru}(u-t)dt
	-\int_{\R}P_{D_+}(t)P_{\Dgaus_{\sigma_\ns}}(u-t)dt\right |du\\
	&\leq (1\pm O(\delta)) \int_{\R}P_{D_+}(t)\int_{\R}\left |P_{\gaus^\tru}(u-t)-P_{\Dgaus_{\sigma_\ns}}(u-t)\right |dtdu\\
	&= (1\pm O(\delta)) \int_{\R}\left |P_{\gaus^\tru}(u)
	-P_{\Dgaus_{\sigma_\ns}}(u)\right |du\\
	&= (1\pm O(\delta))2\dtv(\gaus^\tru,\Dgaus_{\sigma_\ns})\;.
\end{align*}
Therefore, we have that 
$\int_{\R} f(u) du=1\pm O(\dtv(\gaus^\tru,\Dgaus_{\sigma_\ns}))$.
To bound our objective $\dtv(D^\tru_+,D_{t,\epsilon,\psi_+,B_+,\delta}^\alt)$, we have
\begin{align*}
    \dtv(D^\tru_+,D_{t,\epsilon,\psi_+,B_+,\delta}^\alt)
    &=\dtv\left (\left [D^\tru_+\right ]^\bs,\left [D_{t,\epsilon,\psi_+,B_+,\delta}^\alt\right ]^\bs\right )\\
    &=\frac{1}{2}\int_{\R}\left | P_{\left [D^\tru_+\right ]^\bs}(u)
    -P_{\left [D_{t,\epsilon,\psi_+,B_+,\delta}^\alt\right ]^\bs}(u)\right | du\; .
\end{align*}
By the triangle inequality, we can write
\begin{align*}
\dtv(D^\tru_+,D_{t,\epsilon,\psi_+,B_+,\delta}^\alt)& \leq 
    \frac{1}{2}\int_{\R}\left |f(u)-P_{\left [D_{t,\epsilon,\psi_+,B_+,\delta}^\alt\right ]^\bs}(u)\right | du
    +\frac{1}{2}\int_{\R}\left | f(u)-P_{\left [D^\tru_+\right ]^\bs}(u)\right |du\\
    &= O\left (\dtv\left (\gaus^\tru,\Dgaus_{\sigma_\ns}\right )\right )+\frac{1}{2}\left |\int_{\R}f(u) du-1\right |\\
    &= O\left (\dtv\left (\gaus^\tru,\Dgaus_{\sigma_\ns}\right )\right )\\
    &= O\left (\exp\left (-\frac{(c'\eps)^2}{2\sigma_\ns^2} \right )\right )\; .
\end{align*}
From Condition (iv) in Condition \ref{cond:params} and $\sigma_\ns=2(t+\eps)\sigma$, 
we have that
$\frac{(c'\epsilon)^2}{2\sigma_\ns^2}\geq \log (m'/\delta)$};
thus,
$$\dtv(D^\tru_+,D_{t,\epsilon,\psi_+,B_+,\delta}^\alt)=O(\delta/m')\; .$$
The same holds for $D^\tru_-$. 
Therefore, we have that
\begin{align*}
    \dtv(D^\tru,D_\PTF^\alt)
    &= (1-\eta)\dtv(D^\tru_+,D_{t,\epsilon,\psi_+,B_+,\delta}^\alt)
        +\eta\dtv(D^\tru_-,D_{t,\epsilon,\psi_-,B_-,\delta}^\alt)\\
    &=O(\delta/m')\; .
\end{align*}

\newblue{
Before we continue with the rest of the proof, we require the following claim 
about the support of these distributions. 
\begin{claim} \label{app:lem:distribution-support}
	The support of distribution $D'_+$ is 
	$$\bigcup\limits_{i\in \Z_-} [it+(i+1)\eps,it]
	\cup
	\bigcup\limits_{i\in \Z_+} [it,it+(i+1)\eps]\; ,$$
	and the support of distribution $D^\tru_+$ is 
	$$\bigcup\limits_{i\in \Z_-} [it+(i+1)\eps-c'\epsilon,it+c'\epsilon]
	\cup
	\bigcup\limits_{i\in \Z_+} [it-c'\epsilon,it+(i+1)\eps+c'\epsilon]\; .$$
	Similarly, the support of distribution $D'_-$ is \textbf{a subset of} 
	$$\bigcup\limits_{i\in \Z_-} [it+t/2+(i+1)\eps,it+t/2]
	\cup
	\bigcup\limits_{i\in \Z_+} [it+t/2,it+t/2+(i+1)\eps]\; ,$$
	and the support of distribution $D^\tru_+$ is \textbf{a subset of}
	$$\bigcup\limits_{i\in \Z_-} [it+t/2+(i+1)\eps-c'\epsilon,it+t/2+c'\epsilon]
	\cup
	\bigcup\limits_{i\in \Z_+} [it+t/2-c'\epsilon,it+t/2+(i+1)\eps+c'\epsilon]\; .$$
\end{claim}
The above claim directly follows from Claim~\ref{app:clm:ptf-alt-distr}, 
the definition of these distributions, and the fact that the support of $\gaus^\tru$ 
is $[-c'\eps,c'\eps]$.
}

Now we continue with our proof. With our definition of $D^\tru$, 
it remains to prove that there is a PTF $\sgn(p(\cdot))$
satisfying the following:
\begin{enumerate}
    \item $\pr_{(\bx,y)\sim D^\tru}[\sgn(p(\bx))\neq y]=\exp(-\Omega(t^4/\eps^2))$; and
    \item $D^\tru$ satisfies the $O(\eta)$ Massart condition with respect to $\sgn(p(\bx))$.
\end{enumerate}
From the way we defined $D^\tru$ and Property~(ii) 
of Lemma \ref{lem:rejection-alternative-distribution}, 
we know that for $\bu\in \R^n$
\begin{align*}
\pr_{D^\tru}[y=+1\mid\bx=\bu]=(1\pm O(\delta))\pr_{D^\tru}[y=+1\mid\bx^\bs=\bu^\bs]\;.
\end{align*}
\newblue{Since $(1\pm O(\delta))=O(1)$ and $t/\eps=\Omega(1)$},
therefore, it suffices to prove 
these statements on the subspace spanned by $\bs$:
there is a \newblue{degree-$O(t/\eps)$} PTF $\sgn(p(\langle \bs, \cdot \rangle))$ such that
\begin{enumerate}
    \item $\pr_{(x,y)\sim \left [D^\tru\right ]^\bs}[\sgn(p(\langle \bs, x\rangle))\neq y]=\exp(-\Omega(t^4/\eps^2))$; and
    \item $\left [D^\tru\right ]^\bs$ satisfies the $O(\eta)$ Massart condition with respect to $\sgn(p(\langle \bs, x\rangle))$,
\end{enumerate}
where we abuse the notation slightly 
and use $\left [D^\tru\right ]^\bs$ to denote the $1-\eta:\eta$ mixture 
of $[D^\tru_+]^\bs$ and $[D^\tru_-]^\bs$ with $+1$ and $-1$ labels repectively.
We consider a degree-$O(t/\eps)$ PTF $\sgn(p(u))$ such that 
\newblue{$$\sign(p(u))=
\begin{cases}
	+1  &\text{ if } u\in \bigcup\limits_{i\in \Z_-} [it+(i+1)\eps-c'\epsilon,it+c'\epsilon]\cup\bigcup\limits_{i\in \Z_+} [it-c'\epsilon,it+(i+1)\eps+c'\epsilon]\; ;\\
	-1  &\text{ otherwise}\;.
\end{cases}
$$
Notice that according to Claim \ref{app:clm:interval-quantity}, 
the domain with value $+1$ can be written as the union of $O(t/\eps)$ 
many intervals, thus the above
function must be realizable by taking $p$ as a degree-$O(t/\eps)$ PTF.
}

For item 1, \newblue{from Claim~\ref{app:lem:distribution-support}}, 
the support of $+1$ samples ($\left [D_+^\tru\right ]^\bs$) is 
$$\bigcup\limits_{i\in \Z_-} [it+(i+1)\eps-c'\epsilon,it+c'\epsilon]
	\cup\bigcup\limits_{i\in \Z_+} [it-c'\epsilon,it+(i+1)\eps+c'\epsilon]\; ,$$
and the support for -1 samples ($\left [D_-^\tru\right ]^\bs$) is a subset of
$$\bigcup\limits_{i\in \Z_-} [it+t/2+(i+1)\eps-c'\epsilon,it+t/2+c'\epsilon]
	\cup\bigcup\limits_{i\in \Z_+} [it+t/2-c'\epsilon,it+t/2+(i+1)\eps+c'\epsilon]\; .$$
(The above follows from Claim \ref{app:clm:ptf-alt-distr}.)
For $|i+1|\leq t/(2\eps) -1$, these intervals have length at most $t/2-\eps+2c'\eps\newblue{<t/2}$, 
thus do not overlap (for sufficiently small $c'$).
Therefore, this PTF makes no mistake for $i \in [-t/(2\eps),t/(2\eps)-2]$, 
i.e., \newblue{at least for $u\in [-t^2/(2\eps),t^2/(2\eps)-2t]$}.
Thus, its error is at most $\exp(-\Omega(t^2/\eps)^2)$.

For item 2, we show the following:
\begin{enumerate}
	\item [(i)] If $u\in \bigcup\limits_{i\in \Z_-} [it+(i+1)\eps-c'\epsilon,it+c'\epsilon]
	\cup\bigcup\limits_{i\in \Z_+} [it-c'\epsilon,it+(i+1)\eps+c'\epsilon]$ then 
	$$\pr_{(x,y)\sim\left [D^\tru\right ]^\bs}[y=+1|x=u]\geq 1-O(\eta)\; .$$
	\item [(ii)] Otherwise,
	$$P_{\left [D^\tru\right ]^\bs\mid (y=+1)}(u)=0\;.$$
\end{enumerate} 
Item (ii) is straightforward, 
as from Claim~\ref{app:lem:distribution-support}, the support of $D^\tru\mid (y=+1)$ 
(i.e., $D^\tru_+$) is $\bigcup\limits_{i\in \Z_-} [it+(i+1)\eps-c'\epsilon,it+c'\epsilon]
\cup\bigcup\limits_{i\in \Z_+} [it-c'\epsilon,it+(i+1)\eps+c'\epsilon]$.
Since all the support is in item (i), 
in the ``otherwise'' case, it must be the case that 
$P_{\left [D^\tru\right ]^\bs\mid (y=+1)}(u)=0$.

For item (i), we recall that 
$\left [D^\tru\right ]^\bs$ is 
the $1-\eta:\eta$ mixture of $D^\tru_+$ and $D^\tru_-$ 
with $+1$ and $-1$ labels respectively.
Therefore, we need to show for
\begin{align*}
u\in \bigcup\limits_{i\in \Z_-} [it+(i+1)\eps-c'\epsilon,it+c'\epsilon]
\cup\bigcup\limits_{i\in \Z_+} [it-c'\epsilon,it+(i+1)\eps+c'\epsilon] \;,
\end{align*}
we have that
\begin{align*}
P_{D^\tru_+}(u)=\Omega(P_{D^\tru_-}(u))\;.
\end{align*}
Since $D^\tru_+$ and $D^\tru_-$ are \newblue{pointwise close to} the convolution sums of $D'_+$ or $D'_-$ and $\mathcal{N}^\tru$ respectively, 
the above amounts to
\begin{align*}
\newblue{\int_{-c'\eps}^{c'\eps} P_{\gaus^\tru}(u')P_{D'_+}(u-u')du'
=\Omega\left (\int_{-c'\eps}^{c'\eps} P_{\gaus^\tru}(u')P_{D'_-}(u-u')du'\right )}\;.
\end{align*}
Therefore, it suffices to prove for
$u\in \bigcup\limits_{i\in \Z_-} [it+(i+1)\eps-2c'\epsilon,it+2c'\epsilon]
\cup\bigcup\limits_{i\in \Z_+} [it-2c'\epsilon,it+(i+1)\eps+2c'\epsilon]$, 
it holds that
\begin{align*}
P_{D'_+}(u)=\Omega(P_{D'_-}(u))\;.
\end{align*}
We will consider $$u\in \bigcup\limits_{i\in \Z_-} [it+(i+1)\eps-2c'\epsilon,it+2c'\epsilon]
\cup\bigcup\limits_{i\in \Z_+} [it-2c'\epsilon,it+(i+1)\eps+2c'\epsilon]$$
for three different cases:
\begin{enumerate}[leftmargin=*]
	\item [(A)] $|i+1|\leq t/(2\eps)-1$, 
	\item [(B)] $t/(2\eps)-1<|i+1|\leq t/\eps$, and, 
	\item [(C)] $|i+1|> t/\eps$.
\end{enumerate}
For case (A), we begin by recalling that the support of $D'_-$ is
$$\bigcup\limits_{i\in \Z_-} [it+t/2+(i+1)\eps,it+t/2]
	\cup\bigcup\limits_{i\in \Z_+} [it+t/2,it+t/2+(i+1)\eps]\; .$$
For case (A), the support of both $D'_+$ and $D'_-$ are intervals 
with length at most $t/2-\eps$ (see Claim~\ref{app:lem:distribution-support}). 
Thus, the gaps between $D'_+$ intervals and $D'_-$ intervals are at least $\eps$.
Any $u$ for case (A) is at most $2c'\epsilon$ 
far from a $D'_+$ support interval, 
where $c'$ is sufficiently small, 
\newblue{
combining with the fact that gaps between $D'_+$ intervals and $D'_-$ intervals are at least $\eps$, }
this makes $u$ not inside the support of $D'_-$.
Therefore, for any $u$ in case (A), we have that
$P_{D'_-}(u)=0$, which implies $P_{D'_+}(u)=\Omega(P_{D'_-}(u))$.

\newblue{
For case (B), we note by Claim \ref{app:clm:ptf-alt-distr}, we have
$$P_{D'_-}(u)=\frac{\Theta(t)}{\lambda(B_-)}
\sum_{i\in \Z}\frac{1}{|i+1|}\mathbf{1}(u\in it+t/2 +(i+1)(B_- -t/2))\rho_{\sigma_\sn}(u)\; ,$$
and 
$$P_{D'_+}(u)=\frac{\Theta(t)}{\lambda(B_+)}
\sum_{i\in \Z}\frac{1}{|i+1|}\mathbf{1}(u\in it+(i+1)[0,\eps])\rho_{\sigma_\sn}(u)\; .$$
Note that since $t/(2\eps-1)<|i+1|\leq t/\eps$, thus these intervals has lenghth at most $t$, the $D_+$ (resp. $D_-$) intervals do not overlap with other $D_+$ (resp. $D_-$) intervals.  
Therefore taking
$i=\lfloor u/t \rfloor$, 
$$P_{D'_-}(u)=\frac{\Theta(t)}{\lambda(B_-)|i+1|}\mathbf{1}\left (u\in\bigcup\limits_{i\in \Z} it+t/2 +(i+1)(B_- -t/2)\right )\; ,$$
and 
$$P_{D'_+}(u)=\frac{\Theta(t)}{\lambda(B_+)|i+1|}\mathbf{1}\left (u\in\bigcup\limits_{i\in \Z} it+(i+1)[0,\eps]\right )\; .$$
Let $S_-=\bigcup\limits_{i\in \Z} it+t/2 +(i+1)(B_- -t/2)$ and $S_+=\bigcup\limits_{i\in \Z} it+(i+1)[0,\eps]$.
Therefore it suffices to show that 
$$\lambda(B_-)=\Omega(\lambda(B_+))=\Omega(\eps)\; ,$$
and for any $u$ in case (B)
$$u\in S_-\text{ implies }u\in S_+\; .$$
$\lambda(B_-)=\Omega(\eps)$ follows from the definition of $B_-$, 
since
\begin{align*}
	    B_-\eqdef [t/2,t/2+\epsilon]
	    &-\bigcup\limits_{i={\frac{t}{2\eps}-1}}^{\frac{t}{\eps}-1}g([it-2c'\epsilon,it])
	    -\bigcup\limits_{i={\frac{t}{2\eps}-1}}^{\frac{t}{\eps}-1} 
	    g([it+(i+1)\epsilon,it+(i+1)\epsilon+2c'\epsilon])\\
	    &-\bigcup\limits_{i={-\frac{t}{\eps}}-1}^{-\frac{t}{2\eps}-1} 
	    g([it+(i+1)\epsilon-2c'\epsilon,it+(i+1)\epsilon])
	    -\bigcup\limits_{i={-\frac{t}{\eps}}-1}^{-\frac{t}{2\eps}-1} 
	    g([it,it+2c'\epsilon])\; .
\end{align*}
Therefore, recall the definition of the mapping $g$, we get that 
\newblue{
$$\lambda(B_-)
\geq \epsilon-4\sum\limits_{i={\frac{t}{2\eps}}}^\frac{t}{\eps} \frac{2c'\epsilon}{|i|}
\geq \eps-4\frac{t}{\eps} \frac{2c'\epsilon}{\frac{t}{2\eps}}
=\eps-16c'\eps
=\Omega(\eps)\;,$$	}
where the last equality follows from the fact that 
$c'$ is a sufficiently small constant.
To prove for any $u$ in case (B)
$$u\in S_-\text{ implies }u\in S_+\; .$$
We prove the contrapositive, let $u$ be in case (B) and $u\not\in S_+$, then calculations show
\begin{align}
u\in	    &\bigcup\limits_{i={\frac{t}{2\eps}-1}}^{\frac{t}{\eps}-1}g([it-2c'\epsilon,it])
	    \cup\bigcup\limits_{i={\frac{t}{2\eps}-1}}^{\frac{t}{\eps}-1} 
	    g([it+(i+1)\epsilon,it+(i+1)\epsilon+2c'\epsilon])\\
	    &\cup\bigcup\limits_{i={-\frac{t}{\eps}}-1}^{-\frac{t}{2\eps}-1} 
	    g([it+(i+1)\epsilon-2c'\epsilon,it+(i+1)\epsilon])
	    \cup\bigcup\limits_{i={-\frac{t}{\eps}}-1}^{-\frac{t}{2\eps}-1} 
	    g([it,it+2c'\epsilon])\; .
\end{align}
Notice that those intervals are exactly the intervals we carved out. 
Thus, $u\not \in S_-\; .$
\newblue{(We note that the intuition behind this is the following.
For the interval $[g(a),g(b)]$ we carved out on $B-$, 
we made $[a ,b]$ missing from the support of $D'_-$. 
So what we did can be thought of as carving these intervals 
to make the supports of $D'_+$ and $D'_-$ having $2c'\eps$ gaps 
between each other for $t/(2\eps)-1<|i+1|\leq t/\eps$. 
This ensures that after applying the $\gaus^\tru$ noise, 
their supports still do not overlap.)}
This completes the proof for case~(B).}

It remains to analyze case (C).
By Claim~\ref{app:clm:ptf-alt-distr},
we have that
\begin{align*}
    \frac{P_{D'_+}(u)}{P_{D'_-}(u)}
    &=\frac{\frac{\Theta(t)}{\lambda(B_+)}
    \sum_{i\in \Z}\frac{1}{|i+1|}\mathbf{1}(u\in it+\psi_+ +(i+1)(B_+-\psi_+))\rho_{\sigma_\sn}(u)}
    {\frac{\Theta(t)}{\lambda(B_-)}
	\sum_{i\in \Z}\frac{1}{|i+1|}\mathbf{1}(u\in it+\psi_- +(i+1)(B_- - \psi_-))\rho_{\sigma_\sn}(u)}\\
	&=\Omega\left (
	\frac{\sum_{i\in \Z}\frac{1}{|i+1|}\mathbf{1}(u\in it+\psi_+ +(i+1)(B_+-\psi_+))}
    {\sum_{i\in \Z}\frac{1}{|i+1|}\mathbf{1}(u\in it+\psi_- +(i+1)(B_- - \psi_-))}
	\right )\; ,
\end{align*}
where the second equality follows from $\lambda(B_-)=\Omega(\lambda(B_+))=\Omega(\eps)$.
\newblue{
Since $B_-$ is a subset of $[t/2,t/2+\eps]$, thus
$$\sum_{i\in \Z}\frac{1}{|i+1|}\mathbf{1}(u\in it+\psi_- +(i+1)(B_- - \psi_-))\leq \sum_{i\in \Z}\frac{1}{|i+1|}\mathbf{1}\left (u\in it+t/2 +(i+1)\big (\left [t/2, t/2+\eps\right ] -t/2\big )\right )\; ,$$
therefore}
$$	\frac{P_{D'_+}(u)}{P_{D'_-}(u)}= \Omega\left (
	\frac{\sum_{i\in \Z}\frac{1}{|i+1|}\mathbf{1}(u\in it +(i+1)[0,\eps])}
    {\sum_{i\in \Z}\frac{1}{|i+1|}\mathbf{1}(u\in it+t/2 +(i+1)([t/2, t/2+\eps] - t/2))}
	\right )\; .$$
Thus we just need to show the following claim to finish case~(C).
\begin{claim} \label{app:clm:far-side-massart} 
\newblue{
For 
$$u\in \bigcup\limits_{i\in \{\Z_-:|i+1|> t/\eps\}} [it+(i+1)\eps-2c'\epsilon,it+2c'\epsilon]
\cup\bigcup\limits_{i\in \{\Z_+:|i+1|> t/\eps\}} [it-2c'\epsilon,it+(i+1)\eps+2c'\epsilon]\; ,$$
we have that
$$\frac{\sum_{i\in \Z}\frac{1}{|i+1|}\mathbf{1}(u\in it +(i+1)[0,\eps])}
{\sum_{i\in \Z}\frac{1}{|i+1|}\mathbf{1}(u\in it+t/2 +(i+1)([t/2, t/2+\eps] - t/2))}=\Omega(1)\; .$$}
\end{claim}

The idea is to determine for both the numerator and denominator which values of $i$ cause the corresponding indicator function to be 1.
Notice for $(-\infty,-t^2/\eps-t)\cup(t^2/\eps-t,\infty)$, the indicator function in the numerator $it+(i+1)[0,\eps/2]$
will have length at least $t$.
Thus for $u\in (-\infty,-t^2/\eps-2t+\newblue{2c'\eps})\cup (t^2/\eps-\newblue{2c'\eps},\infty)$, 
$u$ must be in at least one of these intervals in the numerator. In particular, there must be $j\leq k$ such that the non-zero terms in the numerator correspond to the indicator terms with $j\leq i \leq k.$
Therefore, the numerator is $\sum_{i=j}^k \frac{1}{|i+1|}$.

\newblue{
Then we examine the terms for the denominator. 
For convenience, we use $l_i^\numer$ and $r_i^\numer$ (resp. $l_i^\deno$ and $r_i^\deno$) 
to denote the left endpoint and right endpoint of the 
$i$th term in the numerator (resp. denominator).
Since the $j-1$th term in the numerator is not satisfied, $u>r_{j-1}^\numer$.
Since $r_{j-1}^\numer>r_{j-2}^\deno$, we know that $u>r_{j-2}^\deno$, 
thus the $j-2$th term in the denominator cannot be satisfied. 
Then, since the $(k+1)$-st term in the numerator is not satisfied, 
we know that $u<l_{k+1}^\numer$. 
Since $l_{k+1}^\numer<l_{k+1}^\deno$, 
$u<l_{k+1}^\deno$, thus the $(k+1)$-st term 
in the denominator is not satisfied. 
Now we have that the denominator 
is at most $\sum_{i=j-1}^k \frac{1}{|i+1|}$.}

Then, we just need to prove for any $j,k\in \newblue{\Z}$ and $j\leq k$ 
(since from the $j$-th to $k$-th term 
in the numerator are satisfied, it must be that $k<-1$ or $j>-1$), 
it holds that
$$\frac{\sum_{i=j}^k \frac{1}{|i+1|}}
{\sum_{i=j-1}^k \frac{1}{|i+1|}}=\Omega(1)\;.$$
This is easy to see, 
since the denominator has at most one term more 
than the numerator and all terms are of comparable size.
This completes the proof for Lemma \ref{lem:ptf-alt-distr}.
\end{proof}

\newblue{
Now we prove the following helper lemma for Lemma \ref{lem:ptf-alt-distr}.

\begin{claim} \label{app:clm:interval-quantity}
$$\bigcup\limits_{i\in \Z_-} [it+(i+1)\eps-c'\epsilon,it+c'\epsilon]\cup\bigcup\limits_{i\in \Z_+} [it-c'\epsilon,it+(i+1)\eps+c'\epsilon]$$
can be equivalently written as the union of $O(t/\eps)$ many intervals on $\R$.
\end{claim}
\begin{proof}
We notice that for $|i+1|\geq t/\eps$, the interval terms inside 
the union has length at least $t$, therefore they overlap with each other and cover the whole space.
Thus, we can rewrite the expression as
\begin{align*}
&\bigcup\limits_{i\in \Z_-} [it+(i+1)\eps-c'\epsilon,it+c'\epsilon]\cup\bigcup\limits_{i\in \Z_+} [it-c'\epsilon,it+(i+1)\eps+c'\epsilon]\\
=&(-\infty,-t^2/\eps-t+c'\eps]\\
&\cup\bigcup\limits_{i\in \{\Z_- |i+1|<t/\eps\}} [it+(i+1)\eps-c'\epsilon,it+c'\epsilon]\cup \bigcup\limits_{i\in \{\Z_+ |i+1|<t/\eps\}} [it-c'\epsilon,it+(i+1)\eps+c'\epsilon]\\
&\cup [ -t^2/\eps-t-c'\eps,\infty)\; , 
\end{align*}
which is the union over $O(t/\eps)$ many intervals.
\end{proof}}

\subsection{Proof of Theorem \ref{thm:lwe-to-massart}} \label{app:subsec:main-proof}

\begin{proof}[Proof of Theorem \ref{thm:lwe-to-massart}]
We give a reduction from $\lwe(m,\R_1^n,\{\pm 1\}^n,\gaus_{\sigma},\mod_1)$ to 
$\massart(m',N,\eta,\opt)$. 
Suppose there is an algorithm $A$ for the problem $\massart(m',N,\eta,\opt)$ with $2\eps'$ distinguishing advantage
such that 
(a) if the input is from the alternative hypothesis case, $A$ outputs ``alternative hypothesis'' with 
probability $\alpha$; and 
(b) if the input is from the null hypothesis case, $A$ outputs ``alternative hypothesis'' with probability 
at most $\alpha-2\eps'$.

Given $m$ samples from $\lwe(m,\R_1^n,\{\pm 1\}^n,\gaus_{\sigma},\mod_1)$, 
we run Algorithm 
\ref{algo:reduction} with input parameters $(m',t,\eps,\delta,\eta')$, 
where $\delta$ and $\eta'/\eta$ 
are sufficiently small positive constants.
If Algorithm \ref{algo:reduction} fails, we output ``alternative hypothesis'' 
with probability $\alpha$ 
and ``null hypothesis'' with probability $1-\alpha$. 
Otherwise, Algorithm \ref{algo:reduction} succeeds, 
and we get $m'$ many i.i.d.\ samples $(\bx_1,y_1),(\bx_2,y_2)\cdots,(\bx_{m'},y_{m'})\in \R^n\times\{\pm 1\}$ 
from $D_\PTF$ 
(they are i.i.d.\ according to Observation \ref{obv:reduction-algorithm-iid}).
With $d=c(t/\eps)$, 
where $c$ is a sufficiently large constant,
we apply the degree-$d$ Veronese mapping
$V_d:\R^n\mapsto\R^N$ on the samples 
to get 
$(V_d(\bx_1),y_1),(V_d(\bx_2),y_2)\cdots,(V_d(\bx_{m'}),y_{m'})\in \R^N\times\{\pm 1\}$.
Then we give these samples to $A$ 
and argue that the above process can distinguish $\lwe(m,\R_1^n,\{\pm 1\}^n,\gaus_{\sigma},\mod_1)$ 
with at least $\eps'-O(\delta)$ advantage. 

We let $D_\LTF$ denote the distribution of $(V_d(\bx),y)$, where $(\bx,y)\sim D_\PTF$ 
($D_\PTF$ depends on $\bs$ 
and which case the original LWE samples come from). 
We note that the samples we provide 
to the Massart distinguisher are 
$m'$ many i.i.d.\ samples from 
$D_\LTF$.
We claim that if we can prove the following items, 
then we are done.
\begin{enumerate}[leftmargin=*]
    \item In the alternative (resp. null) hypothesis case, Algorithm \ref{algo:reduction}
    in the above process fails with probability at most $1/2$.
    \item {\bf Completeness:} If the LWE instance is from the alternative hypothesis case, then $D_\LTF$ has 
    at most $O(\delta/m')$ $\dtv$ distance from a distribution $D$ and there is an LTF $h$ such that 
    \begin{enumerate}
        \item $\pr_{(\bx,y)\sim D}[h(\bx)\neq y]=\exp(-\Omega(t^4/\eps^2))$, and
        \item $D$ satisfies the $\eta$ Massart condition with respect to to the LTF $h$.
    \end{enumerate}
    \item {\bf Soundness:} If the LWE instance is from the null hypothesis case, then $R_\opt(D_\LTF)=\Omega(\eta)$. 
\end{enumerate}
Suppose we have proved the above items. 
In the alternative hypothesis case,
the above process outputs ``alternative hypothesis'' with probability at least $\alpha-O(\delta)$,
since $D$ is at most $O(\delta/m')$ in 
$\dtv$ distance from $D_\LTF$.
Then, in the null hypothesis case, the process outputs 
``alternative hypothesis'' with probability at most $\alpha-\eps'$. 
Therefore, the distinguishing advantage is at least $\eps'-O(\delta)$.
Now it remains to prove these three items.

For the first item, recall that for each sample, 
Algorithm \ref{algo:reduction} uses Algorithm 
\ref{algo:sampling} to generate a new sample. 
If Algorithm \ref{algo:sampling} accepts, 
it outputs a new sample. 
According to Lemma \ref{app:lem:sampling-efficiency}, 
it accepts with probability 
$\Omega(\frac{\lambda(B)(t-\psi)}{t^2})$. 
Under our choice of $\psi$, this is at least 
$\Omega(\eps/t)$ (since we have shown that $\lambda(B_+),\;\lambda(B_-)=\Omega(\eps)$ 
in the proof of Lemma 
\ref{lem:ptf-alt-distr} and $\psi$ is either $\psi_+=0$ or $\psi_-=t/2$).
With $m'=\new{c}(\epsilon/t)m$, 
\new{where $c>0$ is sufficiently small} 
and $m(\epsilon/t)^2$ is sufficiently large, 
by applying the Hoeffding bound, one can see
that Algorithm \ref{algo:reduction} succeeds with probability at least $1/2$.

For the second item, suppose that 
the initial LWE samples are 
from the alternative hypothesis case. 
Then, in the proof of Lemma \ref{lem:ptf-alt-distr} 
(at the beginning of the proof), 
we have shown that there exists 
a distribution $D^\tru_\PTF$ 
(denoted by $D^\tru$ in the proof of Lemma \ref{lem:ptf-alt-distr}) within total variation distance
$O(\delta/m')$ from $D_\PTF^\alt$,
and a degree-$d$ PTF $\sgn(p(\bx))$ such that
\begin{enumerate}
    \item $\pr_{(\bx,y)\sim D^\tru_\PTF}[\sgn(p(\bx))\neq y)=\exp(-\Omega(t^4/\eps^2)]$; and
    \item $D^\tru_\PTF$ satisfies the $O(\eta')$ Massart noise condition with respect to $\sgn(p(\bx))$.
\end{enumerate}
With $\eta'/\eta$ being a sufficiently small constant, 
the second item satisfies the $\eta$ Massart noise condition.
Then, letting $D^\tru_\LTF$ denote the distribution of $(V_d(\bx),y)$ where $(\bx,y)\sim D^\tru_\PTF$, 
one can see that $D^\tru_\LTF$ is at most $O(\delta/m')$ in total variation distance from $D_\LTF$.
Let $h$ denote the corresponding LTF such that $h(V_d(\bx))=\sgn(p(\bx))$. 
Then, after the Veronese mapping,
it must be the case that:
\begin{enumerate}
    \item $\pr_{(\bx,y)\sim D^\tru_\LTF}[h(\bx)\neq y]=\exp(-\Omega(t^4/\eps^2))$; and
    \item $D^\tru_\LTF$ satisfies the $\eta$ Massart noise condition with respect to $h(\bx)$.
\end{enumerate}
This gives the second item.

For the third item, if the initial LWE samples are from the null hypothesis case, then according to 
Lemma \ref{app:lem:ptf-nul-distr},
for any $\bu\in \R^n$, 
$$\pr_{(\bx,y)\sim D_\PTF^\nul}[y=+1\mid \bx=\bu]=(1\pm O(\delta))(1-\eta')\; .$$
Therefore, $R_{\opt}(D_\PTF)=\Omega(\eta')=\Omega(\eta)$, thus $R_{\opt}(D_\LTF)=\Omega(\eta)$.
This completes the proof.
\end{proof}

\section{Putting Everything Together: Proof of Main Hardness Result}
\label{app:proof-main-thm}

Applying Theorem \ref{thm:lwe-to-massart} together 
with Lemma \ref{lem:continuous-lwe-hardness} yields our main theorem:

\begin{theorem} \label{thm:main-formal}
Under Assumption \ref{asm:LWE-hardness}, for any $\zeta \in (0, 1)$, 
there exists $\chi > 0$ such that there is no $N^{O(\log^{\chi} N)}$-time algorithm 
that solves $\massart\left (m'=N^{O(\log^{\chi} N)}, N,\eta=1/3,\opt=1/2^{\log^{1 - \zeta} N}\right)$ 
with $1/3$ advantage.	
\end{theorem}

\begin{proof}
We first take $\chi = 0.01\zeta$.
For Lemma \ref{lem:continuous-lwe-hardness}, we take $\beta = 1 - 0.1\zeta$ and $\gamma=-5$. 
For Theorem \ref{thm:lwe-to-massart}, we take 
\begin{enumerate}
    \item $t=n^{-0.5-0.2\zeta}$,
    \item $\epsilon= \Theta(n^{-1.5})$, 
    \item $\delta$ be a sufficiently small constant, 
    \item $\eta=1/3$, and, 
    \item $m'=c(\epsilon/t)m$, where $c$ is a sufficiently small positive constant.
\end{enumerate}
One can easily check that the conditions in Condition \ref{cond:params} are satisfied.
According to our parameters of choice, the remaining parameters for the hardness are:
\begin{enumerate}
	\item $N=n^{O(t/\epsilon)} \leq n^{O(n^{1 - 0.2\zeta})}$,
	\item $\opt= \exp(-\Omega(t^4/\epsilon^2))=2^{-\Omega(n^{1-0.8\zeta})} 
		\leq \newblue{c2^{-n^{1 - \zeta}}
		\leq c2^{-\log^{1 - \zeta} N}}$, \newblue{for any constant c},
	\item \newblue{$m'=c(\epsilon/t)m=\Omega(n^{-\Theta(1-0.2\zeta)}2^{O(n^\beta)})\geq 2^{O(n^{1 - 0.11\zeta})}
		\geq N^{O\left (\log^{\frac{1-0.11\zeta}{1-0.19\zeta}} N\right )}\geq N^{O(\log^{\chi} N)}$, and} 
	\item \newblue{the time complexity lower bound is 
		$$2^{\Omega\left (n^{1 - 0.1\zeta}\right )} = N^{\omega(\log^{\chi} N)}\; .$$}
\end{enumerate}
Therefore, according to Theorem \ref{thm:lwe-to-massart},
there is no $N^{O(\log^{\chi} N)}$-time algorithm that solves 
$\massart(m'=N^{O(\log^{\chi} N)},N,\eta=1/3,\opt=1/2^{\log^{1 - \zeta} N})$ with $1/3$ advantage.
This completes the proof.
\end{proof}

The above theorem gives the following corollary. 
We note that Corollary \ref{crl:massart-hardness} implies 
our informal Theorem \ref{thm:main-inf}, 
and Corollary \ref{crl:massart-hardness} has stronger parameters.

\begin{corollary} \label{crl:massart-hardness} 
Under Assumption \ref{asm:LWE-hardness}, for any $\zeta \in (0, 1)$, there exists $\chi > 0$ 
such that no $N^{O(\log^{\chi} N)}$ time and sample complexity algorithm 
can achieve an error of $2^{O(\log^{1 - \zeta} N)}\cdot \opt$ 
in the task of learning LTFs on $\R^N$  with $\eta=1/3$ Massart noise.
\end{corollary}

\end{document}